\documentclass[letterpaper,11pt]{article}
\usepackage{geometry}
\usepackage{amsthm}
\usepackage{amssymb}
\usepackage{natbib}
\usepackage{amsmath,color}
\usepackage{mathrsfs}  
\usepackage{stmaryrd}
\usepackage[ruled,vlined]{algorithm2e}
\usepackage{bbold}
\usepackage{bbm}
\usepackage{multirow}
\usepackage[table,xcdraw]{xcolor}
\usepackage{subcaption}
\usepackage[utf8]{inputenc}
\usepackage[export]{adjustbox}
\usepackage{wrapfig}
\usepackage{ulem}

\usepackage[toc,page]{appendix}

\RequirePackage[colorlinks,
            linkcolor=red,
            anchorcolor=blue,
            citecolor=blue
            ]{hyperref}

\def\bzero{\mathbf{0}}

\def \0{\boldsymbol{0}}

\def \v{\boldsymbol{v}}
\def \b{\boldsymbol{b}}

\def \A{\mathbf{A}}
\def \B{\mathbf{B}}
\def \C{\mathbf{C}}

\def \E{\mathbf{E}}
\def \G{\mathbf{G}}

\def \H{\mathbf{H}}

\def \O{\mathbf{O}}
\def \Q{\mathbf{Q}}
\def \R{\mathbf{R}}
\def \T{\mathbf{T}}
\def \U{\mathbf{U}}
\def \V{\mathbf{V}}
\def \V{\mathbf{V}}
\def \F{\mathbf{F}}
\def \W{\mathbf{W}}
\def \X{\mathbf{X}}

\def \Y{\mathbf{Y}}
\def \Z{\mathbf{Z}}
\def \bSigma{\mathbf{\Sigma}}
\def \bOmega{\mathbf{\Omega}}

\def\RR{\mathbb{R}}
\def\scrO{{\mathscr  O}}

\def\calV{{\cal  V}}

\def\calS{{\cal S}}

\def \EE {\mathbb{E}}
\def \PP {\mathbb{P}}

\newtheorem{assump}{Assumption}
 
\newtheorem{theorem}{Theorem}
\newtheorem{lemma}[theorem]{Lemma} 
\newtheorem{proposition}[theorem]{Proposition} 
\newtheorem{remark}[theorem]{Remark}

\title{Multi-source Learning via Completion of Block-wise Overlapping Noisy Matrices}
\author{ Doudou Zhou\footnote{douzh@ucdavis.edu, Department of Statistics, University of California, Davis.},\quad Tianxi Cai\footnote{tcai@hsph.harvard.edu, Department of Biostatistics, Harvard T.H. Chan School of Public Health.},\quad and Junwei Lu\footnote{junweilu@hsph.harvard.edu, Department of Biostatistics, Harvard T.H. Chan School of Public Health.}}

\begin{document}
\maketitle

\begin{abstract}
Matrix completion has attracted attention in many fields, including statistics, applied mathematics, and electrical engineering. Most of the works focus on the independent sampling models under which the observed entries are sampled independently. Motivated by applications in the integration of knowledge graphs derived from multi-source biomedical data such as those from Electronic Health Records (EHR) and biomedical text, we propose the {\bf B}lock-wise {\bf O}verlapping {\bf N}oisy {\bf M}atrix {\bf I}ntegration (BONMI) to treat blockwise missingness of symmetric matrices representing relatedness between entity pairs. Our idea is to exploit the orthogonal Procrustes problem to align the eigenspace of the two sub-matrices, then complete the missing blocks by the inner product of the two low-rank components. Besides, we prove the statistical rate for the eigenspace of the underlying matrix, which is comparable to the rate under the independently missing assumption. Simulation studies show that the method performs well under a variety of configurations. In the real data analysis, the method is applied to two tasks: (i) the integrating of several point-wise mutual information matrices built by English EHR and Chinese medical text data, and (ii) the machine translation between English and Chinese medical concepts. Our method shows an advantage over existing methods.
\end{abstract}
\noindent {\bf Keywords:} Low-rank matrix, matrix completion, singular value decomposition, transfer learning.

\section{Introduction}

Matrix completion aims to recover a low-rank matrix given a subset of its entries which may be corrupted by noise \citep{keshavan2010matrix, candes2009exact}. It has received considerable attention due to the diverse applications such as collaborative filtering \citep{hu2008collaborative, rennie2005fast}, recommendation systems \citep{koren2009matrix}, phase retrieval \citep{candes2015phase}, localization in  internet of things networks \citep{pal2010localization, delamo2015designing, hackmann2013cyber}, principal component regression \citep{jolliffe1982note}, and computer vision \citep{chen2004recovering}. An interesting application of matrix completion is to enable integration of knowledge graphs from multiple data sources with overlapping but non-identical nodes. For example, neural word embeddings algorithms \citep{levy2014neural} have enabled generation of powerful word embeddings based on singular value decompositions (SVDs) of a pointwise mutual information (PMI) matrix. When there are multiple data sources corresponding to different corpus, the PMI matrices associated with different corpora (e.g., text from different languages) are overlapping for words that can be mapped across multiple corpus via existing dictionaries. Matrix completion methods can be used to recover the PMI of all words by combining information from these overlapping corpus. Word embeddings derived from the recovered PMI can subsequently be used to translate words from different corpora.

Much recent progress has been made to efficiently complete large scale low rank matrices, especially under uniform sampling when the observed entries are independently and uniformly sampled \citep[e.g.]{keshavan2010matrix,  chen2015fast, candes2010matrix, candes2010power, mazumder2010spectral,chen2015incoherence, keshavan2010matrix, chen2015fast, zheng2016convergence}. Under noiseless and uniform sampling settings, nuclear norm minimization algorithms  \citep{fazel2002matrix, candes2009exact} and singular value thresholding algorithms  \citep{cai2010singular,tanner2013normalized, combettes2011proximal, meka2009guaranteed} have been proposed. To complete a low-rank matrix given only partial and corrupted entries under uniform sampling, greedy algorithms have been shown as effective. As reviewed in \cite{nguyen2019low}, examples of such algorithms include Frobenius norm minimization \citep{lee2010admira}, alternative minimization \citep{haldar2009rank, tanner2016low, wen2012solving}, optimization over smooth Riemannian manifold \citep{vandereycken2013low} and stochastic gradient
descent \citep{koren2009matrix, takacs2007major, paterek2007improving,sun2016guaranteed, ge2016matrix,ge2017no, du2017gradient, ma2018implicit}.

When the observed entries are sampled independently but not uniformly, for example, in the Netflix problem,  \cite{salakhutdinov2010collaborative} proposed a weighted version that performs better than the aforementioned algorithms. The algorithm was further generalized to arbitrary unknown 
sampling distributions with rigorous theoretical guarantees by \cite{foygel2011learning}. \cite{cai2016matrix} proposed a max-norm constrained empirical risk minimization method, which was proved to be minimax rate-optimal with respect to the sampling distributions. Other methods were also proposed such as nuclear-norm penalized estimators \citep{klopp2014noisy}  and max-norm optimization \citep{fang2018max}.

The most important assumption required by most of the existing matrix completion literature is that the observed entries are sampled independently, whether uniformly or not. However, this assumption typically fails to hold for applications that arise from data integration where the missing patterns are often block-wise. Examples of block-wise missingness include integrating multiple genomic studies with different coverage of genomic features \citep{cai2016structured} and combining multiple PMI matrices from multiple corpora in machine translation as discussed above. An illustration of the general patterns of the entrywise missing and block-wise missing mechanisms is presented in Figure \ref{fig: random}. 
\begin{figure}[!ht]
\begin{subfigure}{0.45\textwidth}
\centering
\includegraphics[width=0.6\linewidth, height=0.6\linewidth]{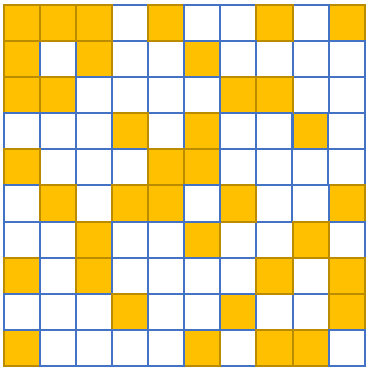} 
\caption{The entrywise missing}
\label{fig:11}
\end{subfigure}
\begin{subfigure}{0.45\textwidth}
\centering
\includegraphics[width=0.6\linewidth, height=0.6\linewidth]{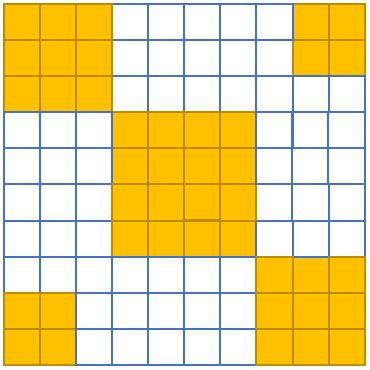}
\caption{The block-wise missing}
\label{fig:12}
\end{subfigure}

\caption{The entrywise missing and block-wise missing patterns in a $10 \times 10$ matrix, where the yellow entries are observed and the white entries are missing.}
\label{fig: random}
\end{figure}

Existing matrix completion methods that can incorporate block-wise missingness largely fall into two categories: (i) imputing the missing blocks for downstream analyses such as prediction \citep{xue2020integrating} and principal component analysis (PCA) \citep{cai2016structured, kxy052}; (ii) utilizing the missing structure for  downstream tasks such as classification \citep{yuan2012multi, xiang2014bi} and prediction \citep{yu2020optimal} without performing imputation. For example, \cite{cai2016structured} proposed a structured matrix completion (SMC) algorithm that leverages the approximate low rank structure to efficiently recover the missing off-diagonal sub-matrix. However, the SMC algorithm considers  a noiseless scenario and does not allow for multi-block missingness structure, which is ubiquitous in the integrative analysis of multi-source or multi-view data. \cite{kxy052} designed an iterative algorithm for the simultaneous dimension reduction and imputation of the data in one or more sources that may be completely unobserved for a sample. However, this approach requires data to be generated from exponential families and no theoretical justifications were provided. Other existing methods mainly focus on the downstream tasks such as classification and feature selection rather than the estimation of the missing blocks and some require additional information such as the class labels \citep{yuan2012multi, xiang2014bi} or some response variable \citep{yu2020optimal} 
of the observations, which is not available in some cases.

To overcome these challenges, we propose the {\bf B}lock-wise {\bf O}verlapping {\bf N}oisy {\bf M}atrix {\bf I}ntegration (BONMI) method under the assumption that the observed entries consist of multiple sub-matrices by sampling rows and columns independently from an underlying low-rank matrix. To be specific, let $\W^\ast$ be the underlying symmetric low-rank matrix. For each source, we observe a principal sub-matrix of $\W^\ast$ with noise, where each row (and the corresponding columns) of the principal sub-matrix is sampled independently with probability $p_0$ from $\W^\ast$. The goal is to estimate the eigenspace of $\W^\ast$. Our idea connects to the orthogonal Procrustes problem \citep{gower2004procrustes, schonemann1966varisim, gower1975generalized}, which has been widely used to align embeddings across languages in the machine translation \citep{kementchedjhieva-etal-2018-generalizing, smith2017offline, conneau2017word, sogaard-etal-2018-limitations, xing-etal-2015-normalized}. We use an orthogonal transformation to align the eigenspace of the two sub-matrices through their overlap, then complete the missing blocks by the inner products of the two low-rank components. 
Moreover, we generalize our method to the multiple sources scenario by applying the method to each pair of the sub-matrices. Since our algorithm operates on matrices from any two sources, it is suitable for parallel computing.

Two of the closest works to our paper are \cite{cai2016structured} and \cite{xue2020integrating}. However, BONMI is different from SMC \citep{cai2016structured} in at least three ways. First, BONMI can complete multiple missing blocks, while SMC only treats one off-diagonal missing block. Even in the case of two sources, BONMI shows the obvious difference from SMC \citep{cai2016structured} by using the $3 \times 3$ blocks structure where SMC uses a $2 \times 2$ blocks structure based on the Schur complement as illustrated in Figure \ref{fig:BONMI}.
\begin{figure}[t]
\centering
\begin{subfigure}{0.45\textwidth}
\includegraphics[width=0.9\linewidth, height=0.75\linewidth]{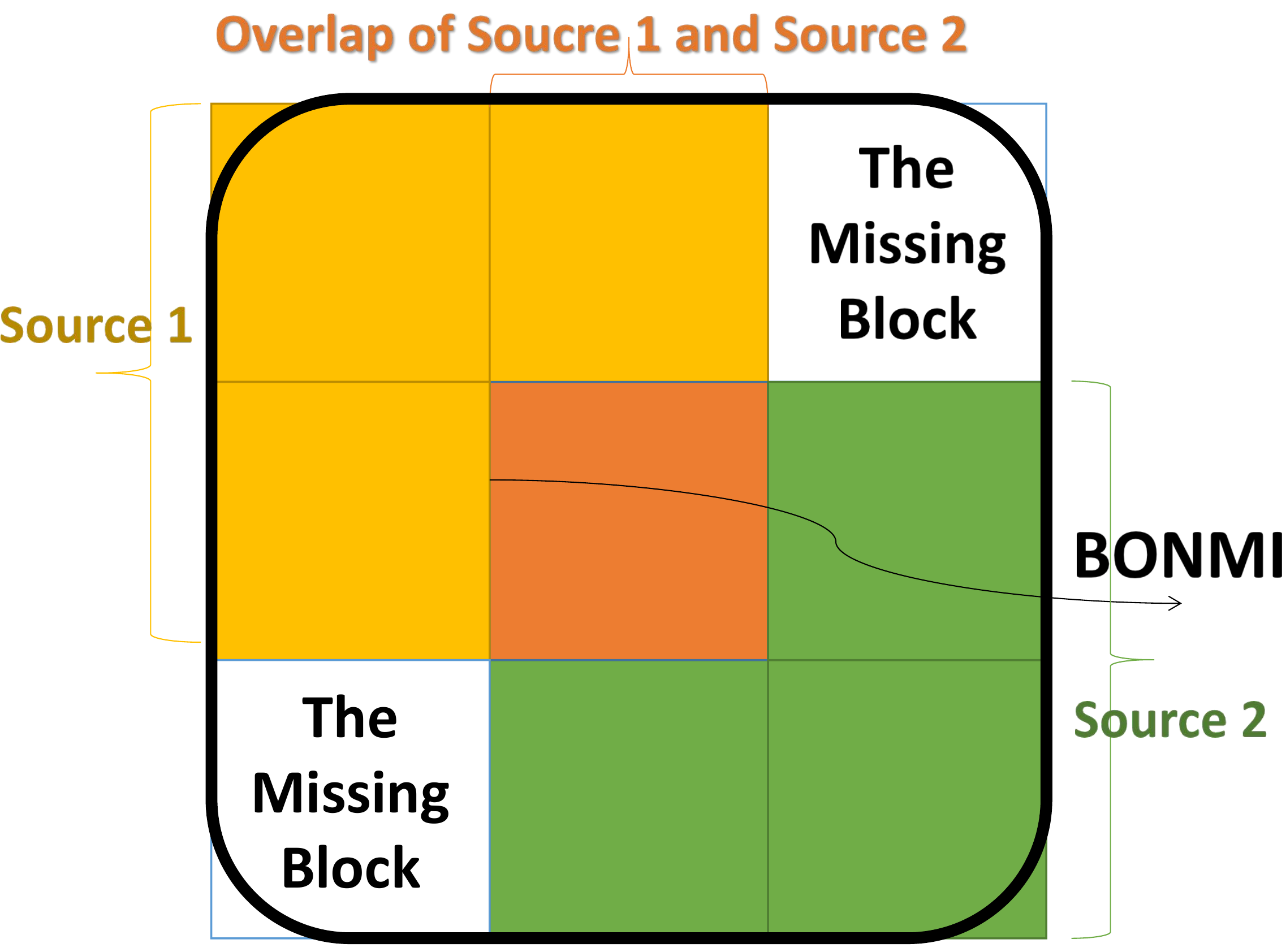} 
\caption{BONMI}
\label{fig:21}
\end{subfigure}
\begin{subfigure}{0.45\textwidth}
\includegraphics[width=0.9\linewidth, height=0.75\linewidth]{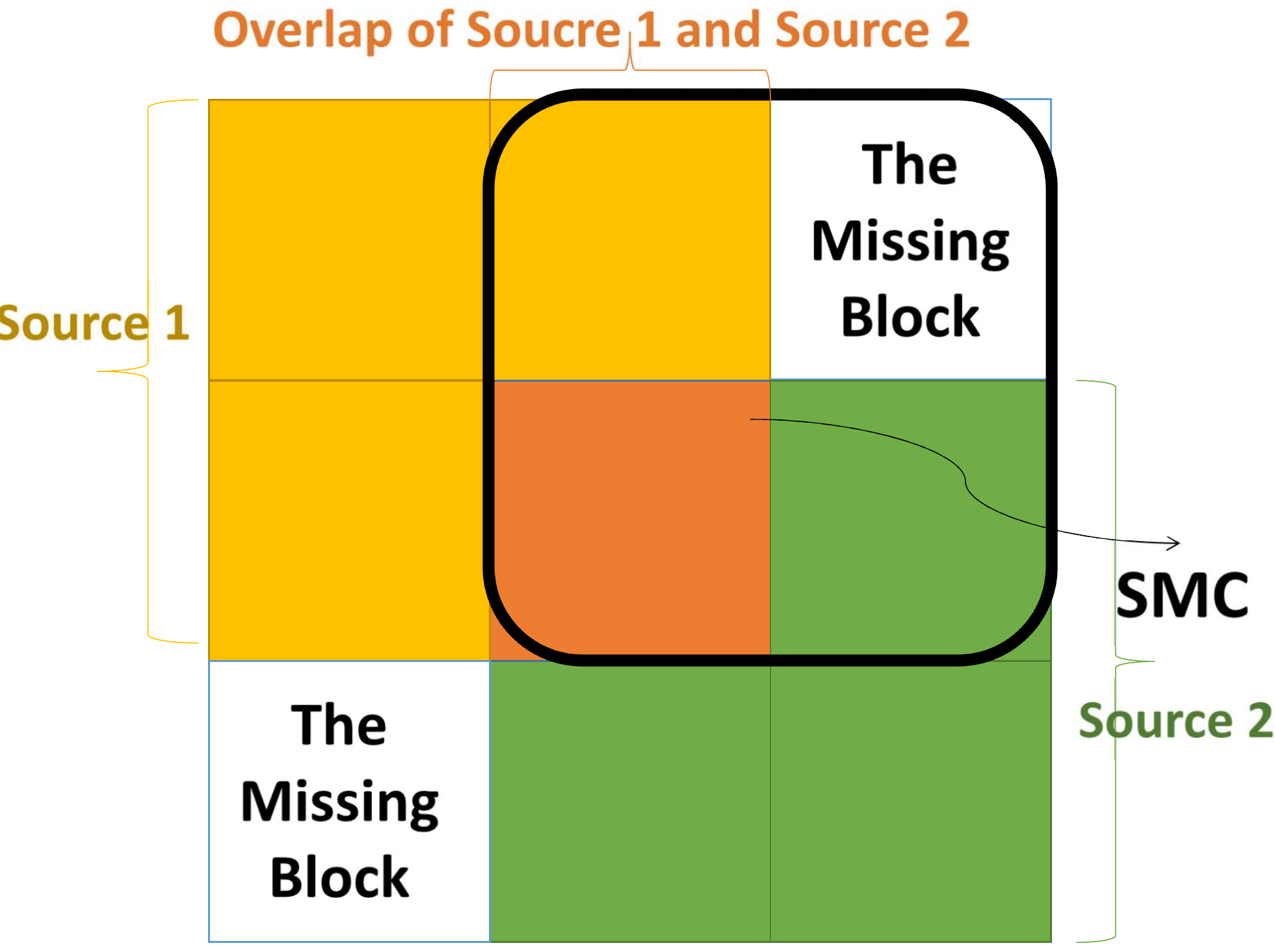}
\caption{SMC}
\label{fig:22}
\end{subfigure}
\caption{BONMI VS SMC. BONMI uses the $3 \times 3$ blocks while SMC uses the $2 \times 2$ blocks.}
\label{fig:BONMI}
\end{figure}
Second, SMC only considers the noiseless case, while BONMI allows the existence of noise. Although both methods achieve perfect recovery under the low-rank and noiseless case, BONMI substantially outperforms SMC under the noisy setting, as demonstrated in both simulations and real data analysis. Intuitively, it is because we exploit the observations thoroughly and avoid doing pseudo-inverse on a noise corrupted matrix. 
Third, we provide the theoretical guarantee for BONMI under the multi-sources and noisy setting, where the statistical rates of the estimator depend on the sampling probability and noise level explicitly. However, the upper bounds for the estimation errors of SMC depend on the missing probability implicit and do not apply to the setting that we consider.  
BONMI is also different from the MBI algorithm proposed by \citep{xue2020integrating} as they solve different problems. \cite{xue2020integrating} focused on the model selection when the covariates were  block-wise missing due to incomplete observations. They assumed a linear model between the response and the covariates and showed the consistency of the estimation of the linear coefficients. Their interest is different from ours where the estimation of the missing blocks itself is of interest and no response variable exits. Moreover, BONMI is dissimilar to MBI by utilizing the low-rankness of the observed matrices, which is a prevailing property of real-world data, such as the PMI matrices. 

Our theoretical results match the state-of-art result of matrix completion under uniform missing \citep{ma2018implicit, chen2015fast, negahban2012restricted, koltchinskii2011nuclear}. To be specific, let $p$ be the entrywise sampling probability under their setting, $p_0$ be the sampling probability of each source under our setting, $N$ be the dimension of the underlying low-rank matrix. When other parameters such as the rank and the condition number are constant, our spectral norm error bound of the underlying low-rank factorization is $O_{P} \big( (2-p_0) \sqrt{N} \big)$ in the case of two sources. Meanwhile, the optimal spectral norm error derived by \cite{ma2018implicit} is $O_{P} ( \sqrt{N/p})$. Under our model, the relation between $p$ and $p_0$ is that $p \approx (2 - p_0^2)/(2-p_0)^2$, so our error bound coincides with their bound. It reveals that even under a dependence sampling mechanism, it is possible to derive a similar error bound under the uniformly independent missing setting. When we have multiple sources, we show how many sources we need to recover enough information from the low-rank matrix while preserving the order of the error bound as the two sources.

 In summary, our paper contributes in three ways. First, we design an efficient algorithm to treat multiple block-wise missing in the matrix completion problem. Second, we propose a way to aggregate the multi-source data optimally. Third, we prove the statistical rate of our estimator, which is comparable to the rate under the independently missing assumption \citep{ma2018implicit, chen2015fast, negahban2012restricted, koltchinskii2011nuclear}.

The rest of the paper is organized as follows. In Section \ref{sec:meth}, we introduce in detail the proposed BONMI method. The theoretical properties of the estimators are analyzed in Section \ref{sec:thm}. Simulation results are shown in Section \ref{sec:simulation} to investigate the numerical performance of the proposed method. A real data application is given in Section \ref{sec:data}. Section \ref{sec:dis} extends the model to asymmetric matrices and concludes the paper. 
For space reasons, the proofs of the main results are given in the supplement. In addition, some key technical tools used in the proofs of the main theorems are also developed and proved in the supplement.

\section{Methodology}
\label{sec:meth}

\subsection{Notations}
We first introduce some notations. We use boldfaced symbols to represent vectors and matrices. For any vector $\v$,  $\| \v \|$ denotes its Euclidean norm. For any matrix $\A \in \RR^{d \times q}$, we let   $\sigma_j(\A)$ and $\lambda_j(\A)$ (if $d=q$) denote its respective $j$th largest singular value and eigenvalue. The smallest singular value $\sigma_{\min(m,n)}(\A)$ will be denoted
by $\sigma_{\min}(\A)$. We let $\| \A \|$, $\| \A \|_{\rm F}$, $\| \A \|_{2,\infty}$ and $\| \A \|_{\infty}$ respectively denote the spectral norm (i.e., the largest singular value), the Frobenius norm, the $\ell_2/\ell_{\infty}$ norm (i.e., the largest $\ell_2$ norm of the rows), and the entry-wise $\ell_{\infty}$ norm (the largest magnitude of all entries) of $\A$.  We let $\A_{j,\cdot}$ and $\A_{\cdot,j}$ denote the $j$th row and $j$th column of $\A$, and let $\A(i,j)$ denote the $(i,j)$ entry of $\A$. Besides, we use the symbol $\equiv$ to denote `defined to be.' For any integer $d \ge 1$, we let $[d] \equiv \{1, ..., d\}$. For indices sets $\Omega_1 \subseteq [d]$ and $\Omega_2 \subseteq [q]$, we use $\A_{\Omega_1,\Omega_2}$ to represent its sub-matrix with row indices $\Omega_1$ and column indices $\Omega_2$. 

We let $\scrO^{n \times r}$ represent the set of all $n \times r$ orthonormal matrices. For a sub-Gaussian random variable $\Y$, its sub-Gaussian norm  is defined as $\|\Y\|_{\psi_2} = \inf \{t>0: \EE e^{-\Y^2/t^2} \leq 2\}$. We use the standard notation $f(n) = O (g(n))$ or $f(n) \lesssim g(n)$ to represent $|f(n)| \leq c|g(n)|$ for some constant $c >0$.  

\def\supast{^{\ast}}

\subsection{Model}
\label{intro:mod}

To simplify the presentation, we consider the symmetric matrices first, which is also inspired by the real data example in Section \ref{sec:data}. The general model for asymmetric matrices and the noisy pattern is discussed in Section \ref{sec:gen}. Denote $\W^\ast = \big[\W^\ast(i,j)\big]_{j \in [N]}^{i\in [N]} \in \RR^{N \times N}$ as the underlying symmetric positive semi-definite population matrix associated with $N$ entities with ${\rm rank}(\W^\ast) = r$. 
We assume that the observed matrices are sampled block-wise. To be specific, for the $s$th source, we sample an index set $\calV_s \subseteq [N]$ independently such that for each $i \in [N]$, we assign $i$ to $\calV_s$ with probability $p_s$:  
$$\PP(i \in \calV_s) = p_s \in (0,1), \text{ for } i \in [N], s \in [m] \, .$$
With the index set $\calV_s$, an matrix $\W^s$ will be observed, which is corresponding to a noisy realization of a submatrix of $\W^\ast$.  Specifically, we have
\begin{equation}
    \W^s = \W^\ast_s+ \E^s=  \big[\W^\ast(i,j)\big]_{j \in \calV_s}^{i\in\calV_s} + \E^s, \text{ for } s \in [m] \,,
\label{eq:mod}    
\end{equation}
where 
the entries of $\E^s$ are independent sub-Gaussian noise with variance $\sigma_s^2$. 
Let $\calV^{\ast} = \cup_{s=1}^m \calV_s$, 
then our task is to recover 
$$\W_0^{\ast}= \W^\ast_{ \calV^{\ast}, \calV^{\ast}} =  \big[\W^\ast(i,j)\big]_{j \in \calV^{\ast}}^{i\in \calV^{\ast}} \in \RR^{n \times n}, \quad \mbox{where $n=|\calV^{\ast} |$.}
$$
Without loss of generality, we assume $\calV^{\ast} = [n]$, otherwise we can rearrange the rows and columns of $\W^{\ast}$. In the PMI word embedding example in Section \ref{sec:data}, $\calV_s$ represents the corpus of the $s$th data source, $\calV\supast$ represents the union of the collected corpora, which has size $n$, and $N$ can be the total number of words. 

\begin{remark}
The underlying assumption of model \eqref{eq:mod} is the repeated measurements, e.g., the overlapping parts may be observed multiple times under the current sampling pattern, albeit the existence of sampling error. However, the repeated measurements are not essential for our algorithms, which will be discussed in detail in Section \ref{sec:gen}. 
\end{remark}

It is easy to see that we can not recover $\W_0^{\ast}$ through $\{ \W^s \}_{s \in [m]}$ for any matrix $\W^\ast$, even if the noisy goes to zero. For example, the non-zero entries of $\W^\ast$ concentrate on a few rows or columns. As a result, we need a standard incoherence condition on our population matrix $\W^\ast$ \citep{candes2009exact} which basically assumes information is distributed uniformly among entries. First, let the eigendecomposition of $\W^\ast$ be
\begin{equation}
    \W^\ast = \U^\ast \bSigma^\ast (\U^\ast)^{\top},
\label{eq:edc}    
\end{equation}
where $\U^\ast \in \RR^{N \times r}$ consists of orthonormal columns, and $\bSigma^\ast$ is an $r \times r$ diagonal matrix with eigenvalues in a descending order, i.e., $\lambda_{\max} \equiv \lambda_1 \geq \cdots \geq \lambda_r \equiv \lambda_{\min} >0$. 

\begin{assump}[Incoherence condition]
The coherence coefficient of $\U^\ast$ satisfies $\mu_0 = O(1)$, where 
\begin{equation*}
\mu_0 \equiv \mu(\U^\ast) = \frac{N}{r} \max_{i \in [N]} \sum_{j=1}^r \U^\ast (i,j)^2 .
\end{equation*}
\label{assump:incoherence}
\end{assump}
We also have the requirement for sample complexity to assure recovery. 
\begin{assump}
The sampling probability $p_0 \equiv \min_{s \in [m]} p_s$ satisfies
$$p_0 \geq C \sqrt{\mu_0 r \log N /N}$$ for some sufficiently large constant $C$. Besides, $ \max_{s \in [m]} p_s / p_0 = O(1)$. 
\label{assump: prob}
\end{assump}

\def\subsksk{_{\scriptscriptstyle s \cap k, s\cap k}}
\def\subsksnk{_{\scriptscriptstyle s \cap k, s\backslash k}}
\def\subskkns{_{\scriptscriptstyle s \cap k, k \backslash s}}
\def\subsnksk{_{\scriptscriptstyle s\backslash k, s \cap k}}
\def\subknssk{_{\scriptscriptstyle k\backslash s, s \cap k}}
\def\subsnksnk{_{\scriptscriptstyle s\backslash k, s\backslash k}}
\def\subsnkkns{_{\scriptscriptstyle s\backslash k, k\backslash s}}
\def\subknssnk{_{\scriptscriptstyle k\backslash s, s\backslash k}}
\def\subknskns{_{\scriptscriptstyle k\backslash s, k\backslash s}}

\subsection{The Noiseless Case}

To illustrate the BONMI algorithm, we first consider the noiseless case when $m=2$. To simplify the notations, we denote $s \backslash k \equiv \calV_s \backslash \calV_k$ and $s \cap k \equiv \calV_s \cap \calV_k$ when they are used as the subscripts of a matrix, and recall that $\W^{\ast}_s \equiv \W^\ast_{\calV_s,\calV_s}$. Assume the two sampled sub-matrices are $\W^\ast_s$ and $\W^\ast_k$. Since the singular values are invariant under row/column permutations, without loss of generality, we can rearrange our data matrices such that

\begin{equation}
\W^\ast_s
=\begin{bmatrix}
\W^\ast\subsksnk & \W^\ast\subsnksk\\
\W^\ast\subsnksnk & \W^\ast\subsksk \\
\end{bmatrix} ; \quad \W^\ast_k
= \begin{bmatrix}
\W^\ast\subsksk &  \W^\ast\subskkns\\
\W^\ast\subknssk & 
 \W^\ast\subknskns\\
\end{bmatrix} 
\label{eq: Ws Wk}
\end{equation}
and
\begin{equation}
     \W^{\ast}_0 = \begin{bmatrix}
\W^\ast\subsnksnk & \W^\ast\subsnksk & \W^\ast\subsnkkns\\
\W^\ast\subsksnk & \W^\ast\subsksk & \W^\ast\subskkns \\ 
\W^\ast\subknssnk & \W^\ast\subknssk & \W^\ast\subknskns 
\end{bmatrix}.
\label{eq: Wstar}
 \end{equation}
An illustration of \eqref{eq: Ws Wk} and \eqref{eq: Wstar} can also be found in Figure \ref{fig:BONMI} (a), where $\W^\ast\subsksk$ is the orange block, $\W^\ast_s$ is represented by the yellow and orange blocks, $\W^\ast_k$ is represented by the green and orange blocks, and $\W^\ast\subsnkkns$ and $\W^\ast\subknssnk$ are the missing blocks. 

Our goal is to recover $\W^{\ast}_0$ based on the observed $\W^\ast_s$ and $\W^\ast_k$. This can be achieved by estimating the missing blocks $\W^\ast\subsnkkns$ and $\W^\ast\subknssnk = \W^{\ast \top}\subsnkkns$ by the symmetry of $\W_0^\ast$. As the missing entries are block-wise, theoretical guarantee based on the assumption of independent missing will fail in the current case. Instead, we propose a method based on the orthogonal transformation, which exploits the following proposition. 
\begin{proposition}
\label{prop: exactly recover}
Suppose $\W^{\ast}$ has eigendecomposition \eqref{eq:edc} and satisfies Assumptions \ref{assump:incoherence} and \ref{assump: prob}. Since $\max \big\{ {\rm rank}(\W_s^\ast),{\rm rank}(\W_s^\ast) \big\} \leq {\rm rank}(\W^\ast) = r$, we suppose the 
eigendecompositions of $\W^\ast_s$ and $\W^\ast_k$ are
$$\W^\ast_s = \V_s^{\ast} \bSigma_s^{\ast} (\V_s^{\ast})^\top \;{\rm and} \; \W^\ast_k = \V_k^{\ast} \bSigma_k^{\ast} (\V_k^{\ast})^\top,$$
where $\V_s^{\ast}$ and $\V_k^{\ast}$ are the eigenvectors of $\W^\ast_s$ and $\W^\ast_k$, respectively. We further decompose $\V_s^{\ast}$ and $\V_k^{\ast}$ as $\V_s^{\ast} = ( (\V_{s1}^{\ast})^\top,(\V_{s2}^{\ast})^\top)^\top$, $\V_k^{\ast} =  ( (\V_{k1}^{\ast})^\top,(\V_{k2}^{\ast})^\top)^\top$ with $\V_{s2}^{\ast}, \V_{k1}^{\ast} \in \RR^{|\calV_s \cap \calV_k| \times r}$. 
Then with probability at least $1 - O(1/N^3)$, $\W^\ast\subsnkkns$ in \eqref{eq: Wstar} can be exactly given by
\begin{equation}
    \W^\ast\subsnkkns = \V_{s1}^{\ast} ( \bSigma_s^{\ast})^{1/2} \G( ( \bSigma_s^{\ast})^{1/2} (\V_{s2}^{\ast})^{\top} \V_{k1}^{\ast} (\bSigma_k^{\ast})^{1/2})  ( \bSigma_k^{\ast})^{1/2} (\V_{k2}^{\ast})^{\top} \, ,
\label{eq: Wsk}
\end{equation}
where $\G(\cdot)$ is a matrix value function defined as:
\begin{equation}
    \G(\C) = \H \Z^{\top} \text{ where } \H \bOmega \Z^{\top} \text{is the SVD of } \C 
\end{equation}
for all matrix $\C \in \RR^{r \times r}$. 
\end{proposition}

Proposition \ref{prop: exactly recover} shows that, when there is no noise, 
$\W^\ast\subsnkkns$ can be recovered precisely based on $\W^\ast_s$ and $\W^\ast_k$ with high probability. The proposition can be easily extended to the case when $m>2$. 
In addition, our theoretical analysis shows that the method is robust to small perturbation. 


\subsection{BONMI Algorithm with Two Noisy Matrices}
\label{method}

In the noisy case, we use the above idea but add an additional step of weighted average . Since it is possible to observe the entries of $\W^{\ast}$ more than once due to multiple sources, weighted average is a natural idea to reduce the variance of estimation in the existence of noise. In reality, the heterogeneity always exists which means the noise strength of different sources may be different. As a result, we decide to use the weights inversely proportional to the noise variance. We start with the case $m=2$ again. { Currently, we decompose two overlapping matrices $\W^s\equiv \W^\ast_s + \E^s$ and $\W^k\equiv \W^\ast_k + \E^k $ as follows
\begin{equation*}
\W^s  
= \begin{bmatrix}
\W\subsnksnk^s & \W\subsnksk^s \\
\W\subsksnk^s & \W\subsksk^s \\
\end{bmatrix} , \quad \W^k  
= \begin{bmatrix}
\W\subsksk^k & \W\subskkns^k \\
\W\subknssk^k & \W\subknskns^k \\
\end{bmatrix} , \quad
\mbox{for $1 \le s < k \le m$.}
\end{equation*}
Then we can combine $\W_s$ and $\W_k$ to obtain
\begin{equation}
    \widetilde \W = \begin{bmatrix}
\W\subsnksnk^s & \W\subsnksk^s &  \bzero\\
\W\subsksnk^s &\W\subsksk^a & \W\subskkns^k\\
 \bzero & \W\subknssk^k  &  \W\subknskns^k  \\
\end{bmatrix}, 
\label{def: tilde W}
\end{equation} 
where $\W\subsksk^a \equiv \alpha_s  \W\subsksk^s + \alpha_k \W\subsksk^k$ is the weighted average of the overlapping part with $\alpha_i>0,i=s,k$ and $\alpha_s + \alpha_k =1$. The weights should ideally depend on the strength of the noise matrices, $\E^s$ and $\E^k$, to optimize estimation. We detail the estimation of the weights in Section \ref{sec:alg}. To estimate $\W^\ast\subsnkkns$, let 
\begin{equation}
  \widetilde \W_s =  \begin{bmatrix}
\W\subsnksnk^s & \W\subsnksk^s\\
\W\subsksnk^s &\W\subsksk^a\\
\end{bmatrix} \quad {\rm and} \quad \widetilde \W_k =  \begin{bmatrix}
\W\subsksk^a & \W\subskkns^k\\
\W\subknssk^k  &  \W\subknskns^k\\
\end{bmatrix},  
\label{eq:W12}
\end{equation}
and the rank-$r$ eigendecompositions of $\widetilde \W_k$ and $\widetilde \W_s$ be $ \widetilde \V_s \widetilde \bSigma_s \widetilde \V_s^\top$ and $\widetilde \V_k \widetilde \bSigma_k \widetilde \V_k^\top$, respectively. Specifically, $\widetilde \V_s$ and $\widetilde \V_k$ can be decomposed block-wise such that $\widetilde \V_s = (\widetilde \V_{s1}^{\top},\widetilde \V_{s2}^{\top})^{\top}$ and $\widetilde \V_k = (\widetilde \V_{k1}^{\top},\widetilde \V_{k2}^{\top})^{\top}$ where $\widetilde \V_{s2}, \widetilde \V_{k1} \in \RR^{|\calV_s \cap \calV_k| \times r}$. So the estimate of $\W^\ast\subsnkkns$ is
\begin{equation}
    \widetilde \W_{s k} =  \widetilde \V_{s1} \widetilde \bSigma_s^{1/2} \G( \widetilde \bSigma_s^{1/2} \widetilde \V_{s2}^{\top} \widetilde \V_{k1}  \widetilde \bSigma_k^{1/2})   \widetilde \bSigma_k^{1/2} \widetilde \V_{k2}^\top,
\label{eq: Wsk noise}    
\end{equation}
according to the Proposition \ref{prop: exactly recover}. After getting $\widetilde \W_{s k}$, we impute it back to $\widetilde \W$ to obtain
\begin{equation}
    \widehat \W = \begin{bmatrix}
\W\subsnksnk^s & \W\subsnksk^s &   \widetilde \W_{s k}\\
\W\subsksnk^s &\W\subsksk^a & \W\subskkns^k\\
\widetilde \W_{s k}^\top & \W\subknssk^k  &  \W\subknskns^k  \\
\end{bmatrix}.
\label{def: hat W}
\end{equation}
}
Then we can obtain the rank-$r$ eigendecomposition of $\widehat \W$, denoted as $\widehat \W_r = \widehat \U \widehat \bSigma \widehat \U^\top$, as an estimate of $\W_0^{\ast}$.

\subsection{BONMI Algorithm}
\label{sec:alg}

We next introduce the BONMI algorithm for recovering $\W^\ast_0$, based on $m \geq 2$ noise-corrupted principal sub-matrices $\{ \W^s  \}_{s \in [m]}$ of $\W^\ast$. Our algorithm consists of three main steps: (I) aggregation of the $m$ matrices, (II) estimation of missing parts, and (III) low-rank approximation, as summarized in Algorithm \ref{alg:OMC}.

\begin{algorithm}[!ht]
\SetAlgoLined

{\bf Input:} $m$ symmetric matrices $\{\W^s\}_{s \in [m]}$ and the corresponding index sets $\{\calV_s\}_{s \in [m]}$; the rank $r$; $n = |\cup_{s=1}^m \calV_s|$\; 

{\bf Step I (a) Estimation of weights:} \For{$1\leq s \leq m$}{
  Let $\widehat \U_{s} \widehat \bSigma_s (\widehat \U_{s})^{\top}$ be the rank-$r$ eigendecomposition of $\W^s$. Estimate $\sigma_s$ by 
  \begin{equation}
      \widehat \sigma_s =  |\calV_s|^{-1} \| \W^s - \widehat \U_{s} \widehat \bSigma_s (\widehat \U_{s})^{\top} \|_{\rm F};
\label{def: sigmas}
  \end{equation}
} 

{\bf Step I (b) Aggregation:} Create $\widetilde \W \in \RR^{n \times n}$ by  \eqref{eq:Wij_over}.

{\bf Step II (a) Spectral initialization:} \For{$1\leq s \leq m$}{
  Let $\widetilde \V_{s} \widetilde \bSigma_s \widetilde \V_{s}^{\top}$ be the rank-$r$ eigendecomposition of $\widetilde \W_{\calV_s \calV_s}$. 
} 

{\bf Step II (b) Estimation of missing parts:} \For{$1\leq s < k \leq m$}{
  Obtain $\widetilde \W_{s k}$ using $\widetilde \V_{s}, \widetilde \bSigma_s, \widetilde \V_{k}, \widetilde \bSigma_k$ by \eqref{eq: Wsk noise}. If a missing entry $(i,j)$ is estimated by multiple pairs of sources $(s,k)$, choose the one estimated by the pair with the smallest $\widehat \sigma_s + \widehat \sigma_k$. Denote the imputed matrix as $\widehat \W$. }

 {\bf Step III Low rank approximation:} Obtain the rank-$r$ eigendecomposition of $\widehat \W$: $\widehat \W_r = \widehat \U \widehat \bSigma \widehat \U^\top$. 
 
 {\bf Output: } $\widehat \U$, $\widehat \bSigma$.
 \caption{{\bf B}lock-wise missing {\bf E}mbedding {\bf L}earning In{\bf T}egration (BONMI).}
\label{alg:OMC}
\end{algorithm}

\def\calS{\mathcal{S}}

{\it Step I: Aggregation.} We first aggregate $\{\W^s \}_{s \in [m]}$ to obtain $\widetilde \W$ similar to the $m=2$ case, which requires an estimation for the weights $\{\alpha_s\}_{s\in [m]}$. Similar to standard meta-analysis, the optimal weight for the $s$th source can be chosen as $\sigma_s^{-2}$. We estimate $\sigma_s$ as by $\widehat \sigma_s =  |\calV_s|^{-1} \| \W^s - \widehat \U_{s} \widehat \bSigma_s \widehat \U_{s}^{\top} \|_{\rm F}$, where $\widehat \U_{s} \widehat \bSigma_s \widehat \U_{s}^{\top}$ is the rank-$r$ eigendecomposition of $\W^s$. We then create the matrix $\widetilde \W \in \RR^{n \times n}$ as follows
 \begin{equation}
    \widetilde \W(i,j) = \sum_{s=1}^m \alpha^s_{i j} \W^s(v_i^s, v_j^s)  \mathbbm{1}(i,j \in \calV_s) ,
\label{eq:Wij_over}    
\end{equation}
for all pairs of $(i,j)$ such that $\calS_{ij} \equiv \sum_{s=1}^m \mathbbm{1}(i, j \in \calV_s) > 0$, where $v_i^s$ denotes the row(column) index in $\W^s$ corresponding to the $i$th row(column) of $\W_0^\ast$, and $$\alpha^s_{ij} = \frac{1}{\hat \sigma_s^2}\big( \sum_{k=1}^m  \mathbbm{1}(i,j \in \calV_k) \hat \sigma_k^{-2} \big)^{-1}.$$
The entries in the missing blocks with $\calS_{ij} = 0$ are initialized as zero. 

\begin{remark}
It is natural to use the inverse of noise variances as the weights to aggregate multiple observations, for instance, weighted least squares \citep{ruppert1994multivariate}. Here we follow the same routine, and the choice is direct since, intuitively, in this way, we can minimize the variance of the noise of the overlapping matrices. A formal analysis is provided in Section \ref{sec: completion error}. 
\end{remark}

{\it Step II: Imputation.}  We next impute the missing entries with  $\calS_{ij} = 0$. For $1 \le s \le k \le m$, we impute the entries of $\widetilde\W$ corresponding to $(\calV_s \backslash \calV_k) \times (\calV_k \backslash \calV_s)$ using $\widetilde \W_s \equiv \widetilde \W_{\calV_s, \calV_s}$ and $\widetilde \W_k$ the same way as \eqref{eq: Wsk noise}. { If a missing entry $(i,j)$ can be estimated by multiple pairs of sources $(s,k)$, we choose the one estimated by the pair with the smallest $\widehat \sigma_s + \widehat \sigma_k$.} After Steps I and II, all missing entries of $\widetilde \W$ are imputed, and we denote the imputed matrix as $\widehat \W$. 

{\it Step III: Low-rank approximation.} Finally, we factorize $\widehat \W$ by rank-$r$ eigendecomposition to obtain the final estimator: $\widehat \W_r \equiv \widehat \U \widehat \bSigma \widehat \U^\top$.

\begin{remark}[Computational complexity] The main computational cost of our algorithm is the eigendecomposition, which is $O(|\calV_s|^2r)$ for the source $s$. At the estimation step, matrix multiplication and the SVD of a $r \times r$ matrix are required, and the computational cost is bounded by $O(|\calV_s| |\calV_k| r)$. As a result, the total computational cost is 
$$
\textstyle O \big( \sum_{s=1}^m |\calV_s|^2r + \sum_{1 \leq s < k \leq m} |\calV_s| |\calV_k| r \big) = O \big( (\sum_{s=1}^m |\calV_s|)^2 r \big) = O\big(  m^2 n^2 r \big)\, .$$
In comparison, the gradient descent algorithms have the computational complexity 
$$O\big(n^2 r + T n^2 r \big) = O\big( (T+1) n^2r \big)\, ,$$
where $T$ is the iteration complexity dependent on the pre-set precision $\epsilon$. For instance, $T = n/r \log (1/\epsilon)$ \citep{sun2016guaranteed},  $T = r^2 \log (1/\epsilon)$ \citep{chen2015fast}, and  $T =  \log (1/\epsilon)$ \citep{ma2018implicit}. As a result, our algorithm has the computational complexity comparable to these algorithms. 
\end{remark}

\def\bVhat{\widehat{\V}}
\def\bWhat{\widehat{\W}}
\def\Vbb{\mathbb{V}}
\def\Vbbhat{\widehat{\Vbb}}

\begin{remark}
\label{rem: PMI}
When $\W^*$ is the PMI matrix, we may obtain embedding vectors for the $n$ entities of the $m$ sources from the imputed  $\bWhat_r = \widehat \X(\widehat \X)^{\top}$ as $\widehat \X = [\widehat \X_1, \ldots, \widehat \X_n]^{\top}$. Then $\bWhat_r$ can be used for the downstream tasks, specifically, the machine translation. To be specific, for each pair of sources $(s,k)$, we can match the entity $i \in \calV_s \backslash \calV_k$ to some entity $j \in \calV_k$ such that $$j = \arg \max_{l \in \calV_k} {\rm cos} (\widehat \X_i, \widehat \X_l) \text{ where } {\rm cos} (\widehat \X_i, \widehat \X_l) = (\widehat \X_i)^\top \widehat \X_l /\big(\|\widehat \X_i\| \|\widehat \X_l\|\big).$$
If ${\rm cos} (\widehat \X_i, \widehat \X_j)$ is larger than a threshold $c$, we can translate the entity $i$ from the $s$th source (language) to the entity $j$ in the $k$th source (language). We can determine $c$ by either setting a desired sensitivity using a test data or through cross-validation with translated pairs or a specificity which can be approximated by the distribution of cosine similarity of related but not synonymous pairs. 
\end{remark}

\section{Theoretical Analysis}
\label{sec:thm}

In this section, we investigate the theoretical properties of the algorithm. We first present some general assumptions required by our theorems. To this end, we define the condition number $\tau \equiv \lambda_1(\W^\ast)/\lambda_r(\W^\ast) = \lambda_{\max}/\lambda_{\min}$.  Besides, we need conditions to bound the noise strength and the condition number. 

\begin{assump}
The entries of $\E^s$ are independent sub-Gaussian noise with mean zero and sub-Gaussian norm  $\|\E^s(i,j)\|_{\psi_2}$, for $s \in [m]$. Let $\sigma \equiv \max_{s \in [m], i,j \in [|\calV_s|]} \|\E^s(i,j)\|_{\psi_2}$. Then $\sigma$ satisfies
\begin{equation*}
    \sigma \sqrt{N/p_0}  \ll \lambda_{\min}\, .
\end{equation*}
\label{assump: signal to noise ratio}
\end{assump}

\begin{assump}
$\tau \equiv \lambda_1(\W^\ast)/\lambda_r(\W^\ast)=\lambda_{\max}/\lambda_{\min} = O(1)$. Throughout this paper, we assume the condition number is bounded by a fixed constant, independent of the problem size (i.e., $N$ and $r$).
\label{assump3}
\end{assump}

\begin{remark}
Assumptions \ref{assump:incoherence}-\ref{assump3} are standard assumptions in many existing literature \citep{ma2018implicit, chen2015fast, negahban2012restricted, koltchinskii2011nuclear}, whereas different rates are required. Specifically, in Assumption \ref{assump: prob}, we only require the sampling probability to be of the order $O(\sqrt{\log N/N})$, which can tend to zero when the population size tends to infinity. In our setting, the sample size of each source is about $N^2 p_0^2$. Then we have  $N^2 p_0^2 \geq C^2 \mu_0 r N \log N$.  Relatively, \cite{ma2018implicit} require that the sample size satisfies $N^2 p \geq C \mu_0^3 r^3 N \log^3N$ for some sufficiently large constant $C > 0$ where $p$ is the entrywise sampling probability. In Assumption \ref{assump: signal to noise ratio}, the sampling probability $p_0$ and the eigenvalue $\lambda_{\min}$ can vary with $N$. 
Compared to \cite{ma2018implicit}, they require the noise satisfies 
$$
\sigma \sqrt{\frac{N\kappa^{3} \mu_0 r \log ^{3} N }{p}} \ll \lambda_{\min } \, .
$$
Our signal to noise ratio assumption has a same order as theirs up to some constants and log factors since $\mu_0,r$ and $\kappa$ are assumed to be constants. 
\end{remark}

The parameter of interest is $\W_0^{\ast}$ with eigendecomposition 
\begin{equation*}
    \W_0^{\ast} = \U_0^{\ast} \bSigma_0^{\ast} (\U_0^{\ast} )^{\top} = \X^{\ast} (\X^{\ast})^{\top},
\end{equation*}
where $ \X^{\ast} = \U_0^{\ast} ( \bSigma_0^{\ast})^{1/2} \in \RR^{n \times r}$. Because $\W_0^{\ast}$ is a sub-matrix of $\W^\ast$, we know ${\rm rank}(\W_0^{\ast}) \leq {\rm rank}(\W^\ast) = r$. With the assumptions above, we can prove that ${\rm rank}(\W_0^{\ast}) = r$ with high probability.

Let $\widehat \X = \widehat \U \widehat \bSigma^{1/2}$ be the output of Algorithm \ref{alg:OMC} and $K \equiv r \mu_0 \tau$. The upper bound for the estimation errors of $\X^{\ast}$ (and hence $\W_0^{\ast}$) under the special case of $m=2$ is presented in Theorem \ref{theorem: W13}. The proof of the theorem is deferred to Appendix \ref{proof:W13}.

\begin{theorem}
Under Assumptions \ref{assump:incoherence}, \ref{assump: prob}, \ref{assump: signal to noise ratio}, and \ref{assump3}, when $m = 2$, with probability at least $1 - O(N^{-3})$, there exists $\O_X \in \scrO^{r \times r}$ such that 
\begin{itemize}
    \item if $p_0 = o(1/\log N)$ or $p_0$ is bounded away from $0$,  we have \begin{equation}
    \|\widehat \X \O_X - \X^{\ast}\| \lesssim \frac{ \{(1-p_0) K^2   +1 \} K}{\sqrt{\lambda_{\min}}}  \sqrt{N} \sigma;
    \label{rate:2.1}
\end{equation}
    \item otherwise, 
    \begin{equation}
    \|\widehat \X \O_X - \X^{\ast}\| \lesssim \frac{ \{(1-p_0) K^2 (p_0 \log N) +1 \} K}{\sqrt{\lambda_{\min}}}  \sqrt{N} \sigma.
    \label{rate:2.2}
\end{equation}
\end{itemize}
\label{theorem: W13}
\end{theorem}

\begin{remark}
Here we compare our result with the state of art result in matrix completion literature \citep{ma2018implicit} under the random missing condition. However, we should notice that their theorems don't hold under the current missing pattern since their entrywise independent sampling assumption is violated. 
Their operator norm error converges to 
\begin{equation}
    \|\widehat \X \O_X - \X^{\ast}\| \lesssim  \frac{\sigma}{\lambda_{\min} (\W_0^{\ast})} \sqrt{\frac{n}{p}} \|\X^{\ast}\|.
\label{rate:macong}
\end{equation}
Recall that $p$ is the entrywise sampling probability under their setting. Besides, we can show that $p \approx 1 - 2 (p_0-p_0^2)^2/(2p_0 - p_0^2)^2 = (2 - p_0^2)/(2-p_0)^2$, $n \approx N (2 p_0 - p_0^2) \approx N p_0$, $\lambda_{\min} (\W_0^{\ast}) \approx p_0 \lambda_{\min}$ and $ \|\X^{\ast}\| \approx \sqrt{ p_0 r \mu_0 \lambda_{\max}}$ (see the proof of Theorem \ref{theorem: W13}).  As a result, their error bound \eqref{rate:macong} reduces to 
\begin{equation}
    \|\widehat \X \O_X - \X^{\ast}\| \lesssim  \frac{ (2-p_0) \sqrt{K} }{\sqrt{ \lambda_{\min}}} \sqrt{N} \sigma .
\label{rate:ma2}    
\end{equation}
When $p_0 \to 1$, our rate is \eqref{rate:2.1}, which has a difference with \eqref{rate:ma2} in the order of $\sqrt{K}$; when $p_0 \to 0$, our rate is \eqref{rate:2.1} or \eqref{rate:2.2}, which  has the difference with \eqref{rate:ma2} in the order of $K^{5/2} \max \{1, p_0 \log N\}$. It means that our rate is same as theirs up to some constants or log factor, which means that the error bound can be similar even under different sampling scenarios. The additional factor $K$ may be caused by the dependence of the sampling pattern. 
\end{remark}


Based on Theorem \ref{theorem: W13}, we generalize it to $m>2$ sources and derive the following theorem. 

\begin{theorem}
Given $0 < \epsilon <1$, let $m = \lceil \log \epsilon/\log(1-p_0) \rceil$. Under Assumptions \ref{assump:incoherence}, \ref{assump: prob}, \ref{assump: signal to noise ratio}, and \ref{assump3}, with probability at least $1 - O\big( \frac{\log^2\epsilon}{\log^2(1-p_0)N^3} \big)$, we have $n \geq (1-\epsilon)N$ and there exists $\O_X \in \scrO^{r \times r}$ such that 
\begin{itemize}
    \item if $p_0 = o(1/\log N)$ or $p_0$ is bounded away from $0$,  we have 
\begin{equation}
     \|\widehat \X \O_X - \X^{\ast}\|\lesssim \Big \{ 1 + \frac{(1-p_0)  K^2 \log^2 \epsilon}{\log^2(1-p_0)} \sqrt{\frac{p_0}{1- (1-p_0)^m}} \Big\} K \sqrt{\frac{N_0}{\lambda_{\min}} } \sigma ;
\label{rate:3.1}
\end{equation}
    \item otherwise, 
    \begin{equation}
     \|\widehat \X \O_X - \X^{\ast}\|\lesssim \Big \{ 1 + \frac{ (1-p_0)  K^2 (p_0 \log N) \log^2 \epsilon}{\log^2(1-p_0)} \sqrt{\frac{p_0}{1- (1-p_0)^m}} \Big \} K \sqrt{\frac{N_0}{\lambda_{\min}}} \sigma.
\label{rate:3.2}
\end{equation}
\end{itemize}
\label{theorem:multi-sources}
\end{theorem}

\def\cos{{\rm cos}}

\begin{remark}
The above theorem gives us guidance on how many sources we need to recover enough parts of $\W^{\ast}$. The order of $m$ can be $|1/\log(1-p_0)| \approx 1/p_0$ when $p_0$ is small. Besides, compared to \eqref{rate:2.1} and \eqref{rate:2.2}, the rates of \eqref{rate:3.1} and \eqref{rate:3.2} have only difference in the log terms, 
which means that even we choose $m$ of the maximum order above, the rate of our error bounds will not change too much. 
\end{remark}

\begin{remark}
Once we obtain the spectral error bound of $\|\widehat \X \O_X - \X^{\ast}\|$, we can utilize it to construct the bound of the translation accuracy. For example, following Remark \ref{rem: PMI}, we can bound $\| \widehat \X_i - \widehat \X_j \|$  when $\X_i^{\ast} = \X_j^{\ast}$ or $\PP\big({\rm cos}(\widehat \X_i, \widehat \X_j) \ge c\big)$ when $\cos(\X_i^{\ast}, \X_j^{\ast}) \le c_0$ for some $c_0$. According to the translation procedure, the translation accuracy can depend on the bound $\sup_{l \in \calV_k} | \cos(\widehat\X_i , \widehat \X_l) - \cos(\X_i^*, \X_l^*)|$ where $i \in \calV_s  \backslash \calV_k$. To bound the quantity, we need additional assumptions on the structures of the underlying matrix $\W^{\ast}$. For instance, if the entities $i$ and $j$ are synonym or translated pairs in different languages, then $\cos(\X_j^*, \X_i^*) > c_1$ for some constant $c_1$, otherwise, $\cos(\X_j^*, \X_i^*) \le c_0$ for some $c_0 < c_1$. 
\end{remark}

\section{Simulation}
\label{sec:simulation}
In this section, we show results from extensive simulation studies that examine the numerical
performance of Algorithm \ref{alg:OMC} on randomly generated matrices for various values of $p_0$, $m$ and $\sigma$. 

\subsection{Comparable Methods}
We compare with SMC \citep{cai2016structured} and two state-of-the-art matrix completion algorithms for the uniform sampling setting, alternating least squares (ALS) \citep{hastie2015matrix} and vanilla gradient descent (VGD) \citep{ma2018implicit}. Since SMC can only be applied to complete a single missing block, as illustrated by Figure \ref{fig:BONMI} (b), we use it to complete the missing blocks of each pair of sources. After all missing blocks are imputed, we use the rank-$r$ SVD to obtain the low-rank estimator for SMC. In addition, VGD and ALS can only operate on a single source matrix. To use the two algorithms, we apply them on $\widetilde \W$ created in Step I (b) of Algorithm \ref{alg:OMC}. For all methods, we use the true rank $r$.


Besides, an potential application of BONMI is machine translation. To be specific, in reality, the overlapping parts may not be known fully. For instance, $\{ \W^s \}_{s \in [m]}$ are multilingual co-occurrence matrices or PMI matrices \citep{levy2014neural}, then each vertex is a word and the overlapping parts are created by bilingual dictionaries, which are limited in some low-resource languages and always cover only a small proportion of the corpora. In this case, BONMI can utilize these matrices and their known overlap to train multilingual word embeddings (i.e., $\widehat \X$). For the words not known in the overlapping set, if their embeddings (i.e., rows of $\widehat \X$), are close enough, it means that they have a similar meaning and should be translated to each other. We evaluate the translation precision in the simulation setting (iii). As a baseline, we also compare BONMI to the popular orthogonal transformation method \citep{smith2017offline} which use the single-source low-rank factors $\widehat \X_s = \widehat \U_s \widehat \bSigma_s^{1/2}$ for $s \in [m]$. We denote the method as `Orth'. Another standard approach is to use one data source as pre-training and the new data sources to continue training. This effectively corresponds to imputing the missing blocks of the PMI matrix as zero. We call the method `Pre-trained'.

\subsection{Data Generation Mechanisms and Evaluation Metrics}
Throughout, we fix $N = 25, 000$ and $r=200$ which are compared to our real data. We then generate the random matrix $\W^\ast = \U^\ast \bSigma^\ast (\U^\ast)^\top$, where the eigenvalues of the diagonal matrix $\bSigma^\ast$ are generated independently from the uniform distribution ${\rm U}(\sqrt{N},4\sqrt{N})$. The singular space $\U^\ast$ is drawn randomly from the Haar measure. Specifically, we generate a matrix $\H \in \RR^{N \times r}$ with i.i.d. standard Gaussian entries, then apply the QR decomposition to $\H$ and assign $\U^\ast$ with the Q part of the result. To generate data for the $s$th source, we generate a sequence of independent Bernoulli random variables with success rate $p_0$: $\delta^s = (\delta^s_1,\dots,\delta^s_{N})$ to form the index set $\calV_s = \{i:\delta^s_i = 1,i \in [N]\}$, for $s \in [m]$. We then generate the noise matrix $\E^s \in \RR^{ |\calV_s|  \times |\calV_s|}$
with its upper triangular block including the diagonal elements from normal distribution $N(0,\sigma_s^2)$ and lower triangular block decided by symmetry, where we let the noise level $\sigma_s$ vary across the $m$ sources. 

We consider a total of three settings with the first two focusing on the task of matrix completion and setting 3 focusing on the downstream task of machine translation. For the matrix completion task, we consider two settings: (i) $m=2$, $\sigma_s = 0.1 s$, and let $p_0$ vary from $0.1$ to $0.3$;  (ii) $p_0=0.1$, $\sigma_s = 0.1$, and let $m$ vary from $2$ to $6$. For the machine translation task, we consider the setting (iii) where we let $m=2, 3$, $p=0.1$, $\sigma_s = s\sigma$ and let a noise level $\sigma$ vary from $0.3$ to $0.5$. 
 
To evaluate the performance of matrix completion, we use the relative ${\rm F}$-norm and spectral norm errors of the estimation of $\W_0^{\ast}$ defined as  
\begin{equation*}
   {\rm err}_{\rm F}(\widehat \W, \W_0^{\ast}) = \frac{\|\widehat \W -\W_0^{\ast}\|_{\rm F} }{\|\W_0^{\ast}\|_{\rm F}} \quad {\rm and} \quad {\rm err}_2(\widehat \W, \W_0^{\ast}) = \frac{\|\widehat \W -\W_0^{\ast}\| }{\|\W_0^{\ast}\|}.
\end{equation*}
To evaluate the overall performance of machine translation in setting (iii), we additionally generate test data for evaluation. Specifically, we additionally sample $n_{\rm test} = 2000$ vertices from $\calV \backslash \calV^{\ast}$ where $\calV^{\ast} = \cup_{s=1}^m \calV_s$, denoted as $\calV^{\rm test}$, and combine $\calV^{\rm test}$ and $\calV_s$ to get $\calV_s' = \calV^{\rm test} \cup \calV_s$ as the final vertex set of the $s$th source. We then use $\calV_s'$ to generate $\W^s$. Notice now $\E^s\in \RR^{ |\calV_s'| \times |\calV_s'|}$. However, we treat $\calV^{\rm test}$ as unique vertices across the $m$ sources, which means that we will not combine entries of $\calV^{\rm test}$ in Algorithm $\ref{alg:OMC}$. The role of $\calV^{\rm test}$ is exactly the testing set in machine translation. We average the $m-1$ translation precision from the $s$th source to the $1$th source, $s=2,\dots,m$. The translation precision is defined as follows: for a vertex $i \in \calV^{\rm test}$, we can get its embedding in the $s$th source corresponding to one row in $\widehat \X$, denoted as $\widehat \X_i$\,. Then we find its closest vector $\widehat \X_j$ for $j \in \calV_1'$ with largest cosine similarity as illustraed in Remark \ref{rem: PMI}. If the $j$th vertex in the $1$st source and the $i$th vertex in the $s$th source are the same vertex in $\W^{\ast}$, we treat it as a correct translation. The precision of the $s$th source is the ratio of right translations among the test test in the $s$th source.

\subsection{Results}
We summarize simulation results averaged over 50 replications for settings (i)-(ii) in Figure \ref{fig:setting2} and setting (iii) in Figure \ref{fig:setting3}. BONMI outperforms all competing methods across the three settings. 
\begin{figure}[t]
  \centering
 \begin{tabular}{c}
 \includegraphics[width=0.9\textwidth,height=0.4\linewidth]{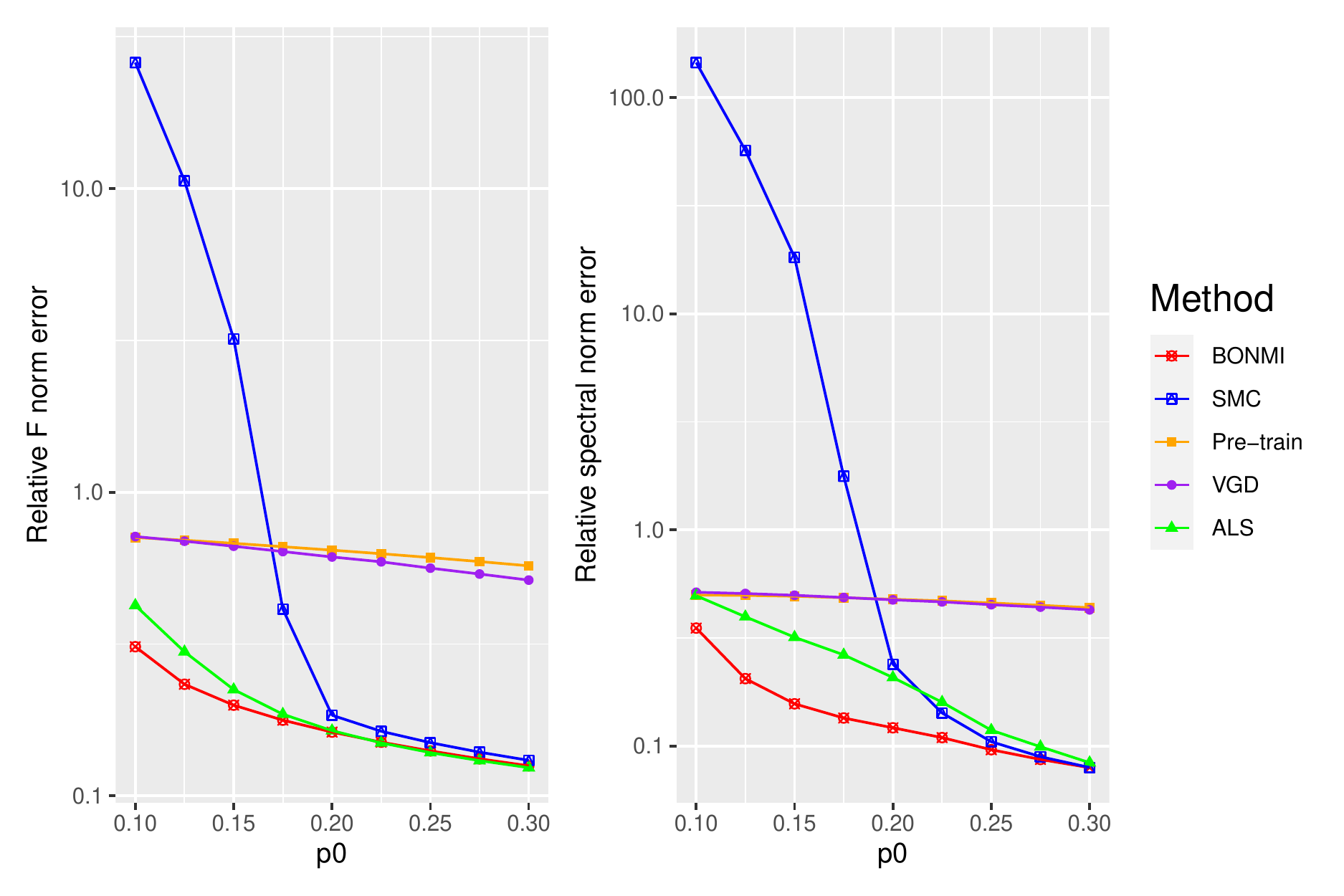}\\
 (a) setting (i): fix $m=2$ and range $p_0$ from $0.1$ to $0.3$. \label{fig:setting1}\\
 \includegraphics[width=0.9\textwidth,height=0.4\linewidth]{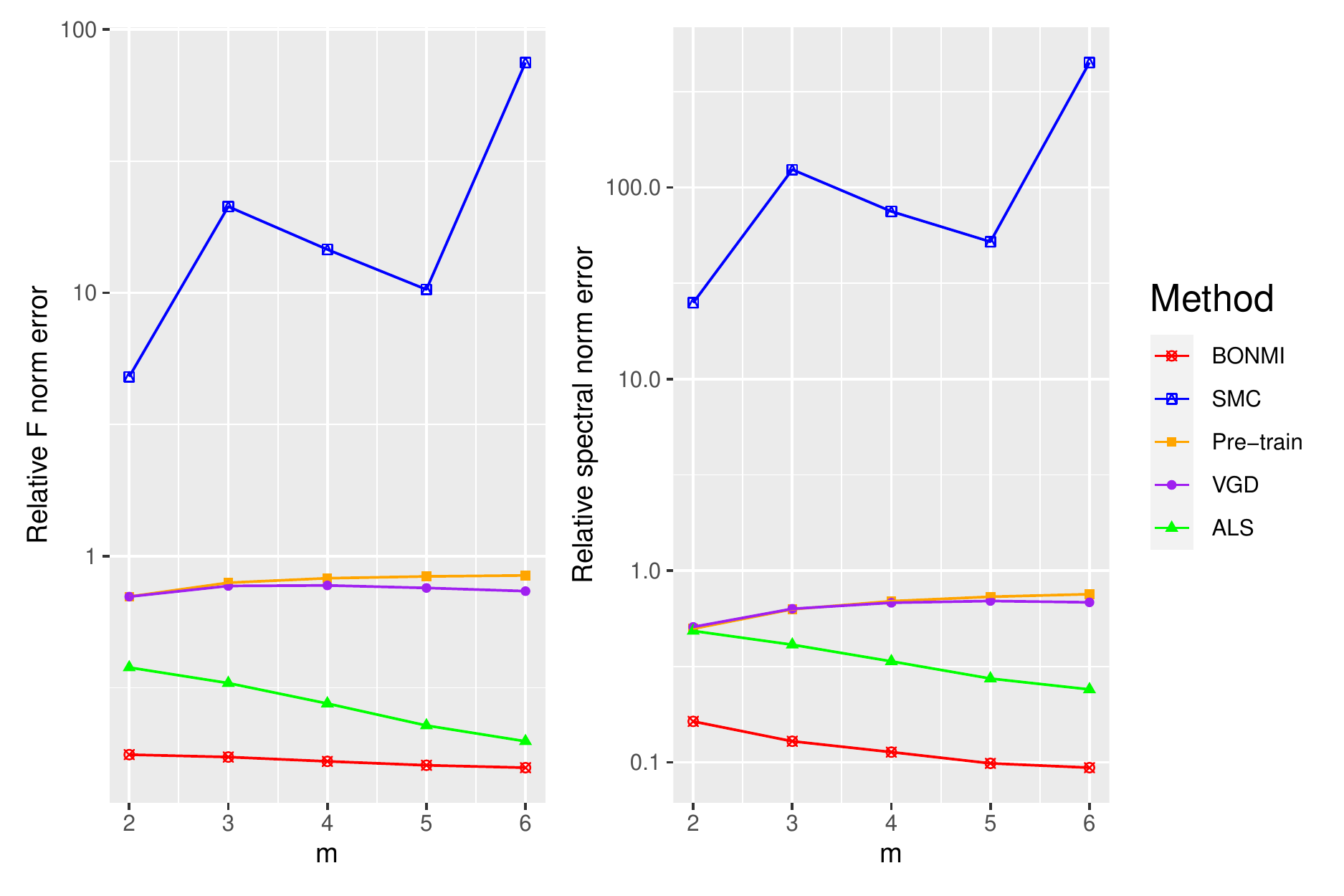} \\
 (b) setting (ii): fix $p_0=0.1$ and range $m$ from $2$ to $6$.
 \end{tabular}
 \caption{Simulation results of settings (i) and (ii). The relative estimation errors of $\W_0^\ast$ are presented.}
\label{fig:setting2}
\end{figure}
In settings (i) and (ii), the results of the ${\rm F}$-norm and spectral norm errors are consistent. In setting (i), we can see that the relative error of all methods decreases when the observation rate $p_0$ increases as expected. The advantage of BONMI in the accuracy of matrix completion is more pronounced when the observation rate $p_0$ is low. When $p_0$ is very small, SMC tends to fail. In setting (ii), the error of BONMI decreases as $m$ increases, which is due to the information gain from multiple sources. However, both the naive pre-training method and VGD do not always perform better as $m$ increases. Overall, BONMI dominates all competing methods across different choices of $m$.  In setting (iii), the translation precision of BONMI, ALS, and pre-training are $100\%$ when the noise strength $\sigma$ is small, but SMC and VGD both perform poorly for this task. As $\sigma$ increases, the performances of BONMI, ALS, and Pre-train all decrease but BONMI decreases less than others.

\begin{figure}[t]
  \centering
 \includegraphics[width=0.9\textwidth,height=0.6\linewidth]{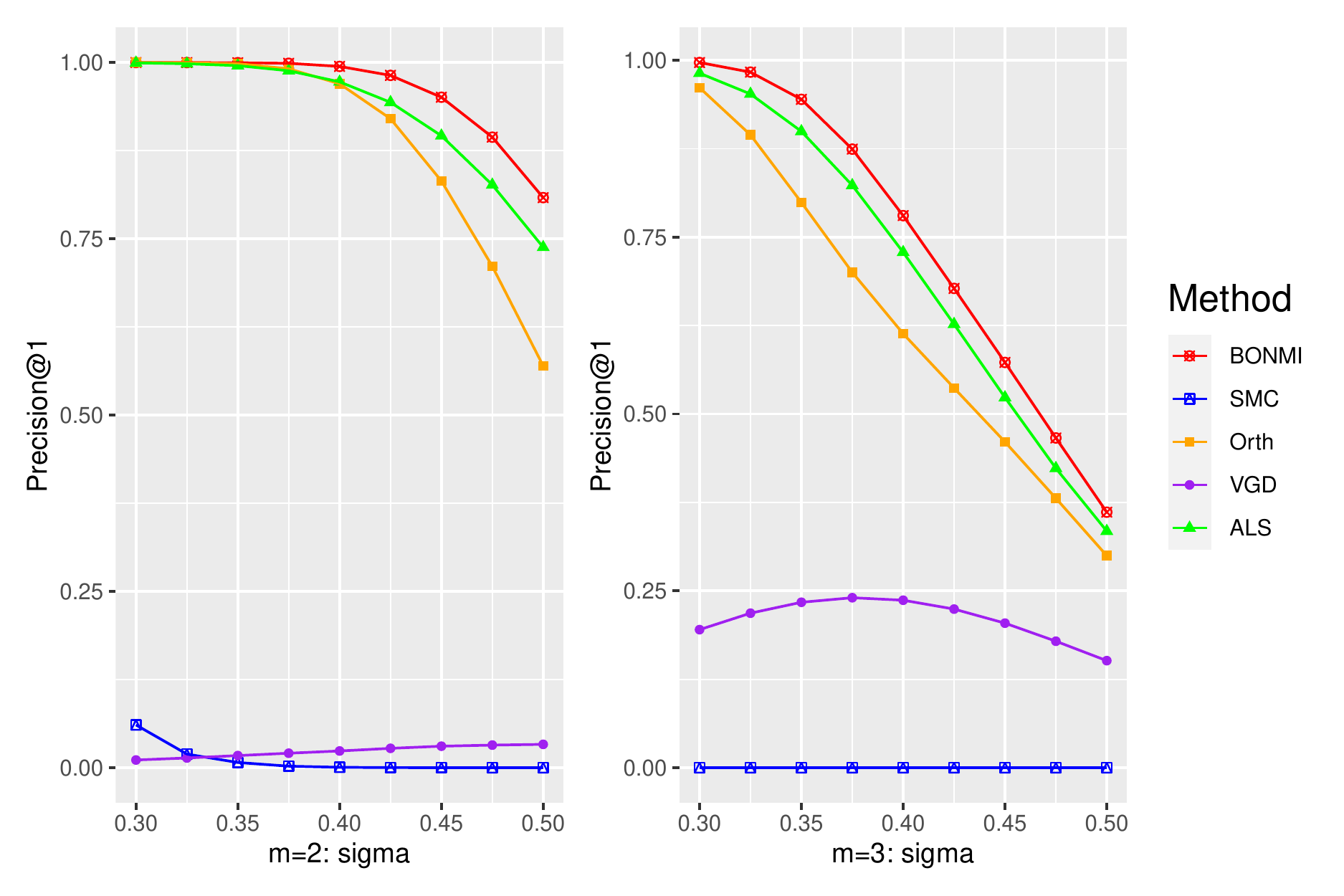} 
 \caption{setting (iii): fix $p_0=0.1$ and range $\sigma$ from $0.3$ to $0.5$. }
\label{fig:setting3}
\end{figure}

\section{Real Data Analysis}
\label{sec:data}

\def \PMI {\rm PMI}
\def \SPPMI {\rm SPPMI}
\def \CUI {\rm CUI}

\begin{figure}[htpb!]
  \centering
 \includegraphics[width=0.4\textwidth,height=0.4\linewidth]{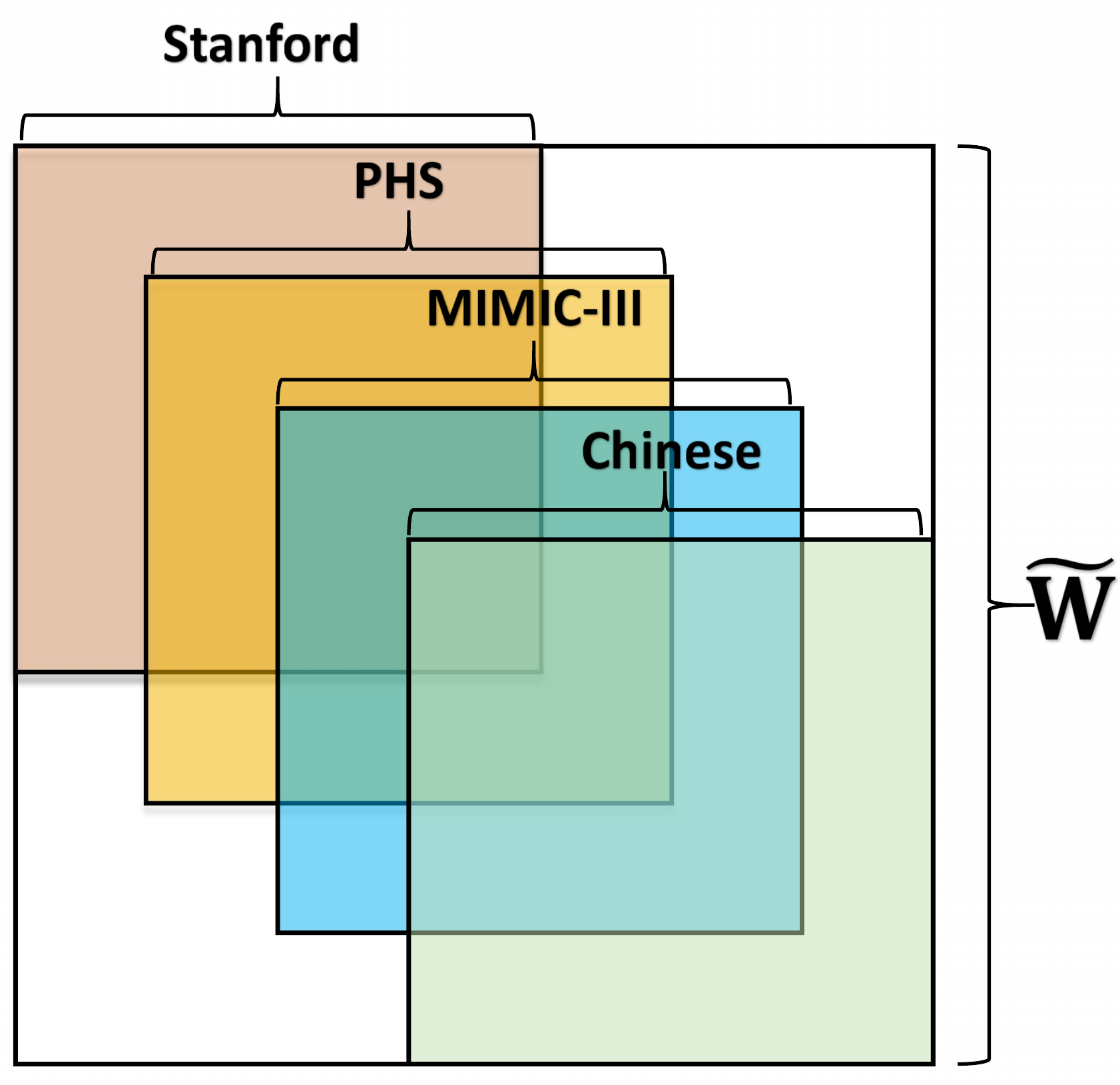}
 \caption{The aggregation of the four PMI matrices. The four sources all have an overlapping with each other source.}
\label{fig:realdata}
\end{figure}

In this section, we apply BONMI 
to obtain clinical concept embeddings using multiple PMI matrices in two different languages, English and Chinese. PMI is an information-theoretic association measure widely used in natural language processing \citep{church-hanks-1990-word, junwei}. Specifically, for two concepts $x$ and $y$, ${\rm PMI}(x,y)$ is defined as 
$$
{\rm PMI}(x,y) = \log   \frac{\PP(x,y)}{\PP(x) \PP(y)} \, ,
$$
where $\PP(x)$ and $\PP(y)$ are the occurrence probabilities of $x$ and $y$, respectively, and $\PP(x,y)$ is the co-occurrence probability of $x$ and $y$. All of them can be calculated from the co-occurrence matrix \citep{levy2014neural,beam2018clinical}. The clinical concepts in English have been mapped to {\em Concept Unique Identifiers} ($\CUI$s) in the Unified Medical Language System (UMLS) \citep{humphreys1993umls}. Factorization of PMI matrices has been shown highly effective in training word embeddings \citep{levy2014neural}. Our goal here is to enable integrating multiple PMI matrices to co-train clinical concept embeddings for both CUIs and Chinese clinical terms.

The input data ensemble consists of three CUI PMI matrices and one Chinese PMI matrix. The three CUI PMI matrices are independently derived from three data sources (i) $20$ million notes at Stanford \citep{finlayson2014building};  (ii) $10$ million notes of $62$K patients at Partners Healthcare System (PHS) \citep{beam2018clinical}; and (iii) health records from MIMIC-III, a freely accessible critical care database \citep{johnson2016mimic}. We choose sub-matrices from these sources by thresholding the frequency of these CUI and keeping those with semantic types related to medical concepts. Finally, we obtain the Stanford PMI with $8922$ CUI, the PHS PMI with $10964$ CUI, and the MIMIC CUI with $8524$ CUI. The mean overlapping CUI of any two sources is $4480$ and the total number of the unique CUI of the three sources is $17963$. 

Multiple sources of Chinese medical text data, such as medical textbooks and Wikipedia, are also collected. We then build a PMI matrix of Chinese medical terms with dimension $8628$. 
A Chinese-English medical dictionary is used to translate these Chinese medical terms to English, which are further mapped to CUI. Finally, we obtain $4201$ Chinese-CUI pairs, and we use $2000$ pairs as the training set (the known overlapping set) and the other $2201$ pairs as the test set to evaluate the translation precision. An illustration of the aggregation of the three CUI PMI matrices and one Chinese PMI matrix is presented in Figure \ref{fig:realdata}. 

\begin{table}[!ht]
\centering
\caption{Results of the integration of four PMI matrices: (a) rank correlation between the pairwise cosine similarity of estimated embedding vectors from the completed matrix  and the similarity or relatedness from human annotation; (b) accuracy in translation based on the estimated embedding vectors.}
\label{tab1}
\centerline{(a) Rank correlation with human annotations}
\begin{tabular}{c|ll|lllll}\hline
{Source} & Relationship & Set & BONMI & Pre-train & ALS & SMC & VGD \\ \hline
\multirow{4}{*}{Chinese-Chinese} 
& Relatedness & I & 0.741 & 0.756 & 0.747 & 0.066 & 0.761\\
& Relatedness & II& 0.661 & 0.659 & 0.659 & 0.327 & 0.663\\
& Similarity & I & 0.707 & 0.724 & 0.704 & 0.105 & 0.731\\
& Similarity & II & 0.716 & 0.728 & 0.718 & 0.271 & 0.726\\ \hline
\multirow{4}{*}{CUI-CUI} 
& Relatedness & I & 0.678 & 0.639 & 0.689 & 0.369 & 0.642\\
& Relatedness & II& 0.604 & 0.598 & 0.609 & 0.141 & 0.592\\
& Similarity & I& 0.614 & 0.600 & 0.608 & 0.243 & 0.582\\
& Similarity & II & 0.634 & 0.635 & 0.643 & 0.171 & 0.622\\ \hline
\multirow{4}{*}{Chinese-CUI} 
& Relatedness & I & 0.671 & 0.406 & 0.630 & 0.306 & 0.413\\
& Relatedness & II & 0.666 & 0.459 & 0.632 & 0.334 & 0.385\\
& Similarity & I & 0.611 & 0.324 & 0.569 & 0.352 & 0.339\\
& Similarity & II & 0.707 & 0.479 & 0.683 & 0.365 & 0.426\\ \hline
\end{tabular}\vspace{.2in}
\centerline{(b) Translation accuracy}
\begin{tabular}{c|l|lllll}\hline
 &  & BONMI & Pre-train & ALS & SMC & VGD \\ \hline
\multirow{3}{*}{Translation Precision}   & @5 & 0.398 & 0.051 & 0.324 & 0.020 & 0.049\\
& @10 & 0.478 & 0.095 & 0.411 & 0.029 & 0.092\\
& @20 & 0.554 & 0.167 & 0.485 & 0.041 & 0.147\\
\hline
\end{tabular}
\end{table}

From each method, we obtain the embedding vectors for all entities by performing an SVD on the imputed PMI matrix as discussed in Section 2. 
To evaluate the quality of the obtained embedding, we compare the cosine similarity of trained embeddings against the gold standard human annotations of the concept similarity and relatedness. We considered two sets of gold standard annotations: (I) relatedness and similarity of $566$ pairs of UMLS concepts in English previously annotated by eight researchers in \cite{30}; and (II) $200$ pairs of Chinese medical terms randomly selected  and annotated by four clinical experts. The $566$ UMLS concepts in English were also translated to Chinese. The Chinese medical terms were also translated into English and mapped to the UMLS CUI. Each concept pair thus can be viewed as CUI-CUI pairs, Chinese-Chinese pairs, and Chinese-CUI pairs. The gold standard human annotation assigns each concept pair a relatedness and similarity score, defined as the average score from all reviewers. For each concept pair, we compare the cosine similarity of their associated embeddings against the human annotations of their similiarity and relatedness. We evaluate the quality of the embeddings based on (i) the rank correlation between the cosine similarity and human annotation; and (ii) the accuracy in translating Chinese medical terms to CUIs in English. The Precision@$k$ is defined similarly as the translation accuracy defined in Section \ref{sec:simulation}. The difference is that when the truth CUI is among the CUIs with the top $k$ largest cosine similarity to the Chinese term, then it is treated as a correct translation. Precision@$k$ is the ratio of the correct translations given a $k$. Here we choose $k$ as $5$, $10$ and $20$.

To choose the rank of the matrix, we analyze the eigen decay of the matrices. The eigen decay has been widely used to determine the rank of low-rank matrices, for example, in principal component analysis \citep{jolliffe2005principal}, word embedding \citep{hong2021clinical} and network analysis \citep{arroyo2021inference}. We calculate the eigen decay of the overlapping sub-matrices of each pair of sources, and choose the rank $r$ that makes the cumulative eigenvalue percentage of at least one of the matrices more than $95\%$, which is $300$. We then use $r = 300$ for all methods. 

\subsection{Results}

We present the results of BONMI and other completing methods in Table \ref{tab1}. We can observe that all methods other than SMC perform similarly when assessing relatedness and similarity of Chinese-Chinese and CUI-CUI since these pairs belong to the same corpus. However, BONMI outperforms other methods when evaluating relatedness and similarity between Chinese-CUI pairs which belong to the missing blocks.  BONMI also attains higher accuracy in the translation task. This suggests that BONMI has the advantage over existing methods in providing embedding vectors that enable accurate assessment of relatedness between entity pairs that do not belong to the same corpus.

\section{Discussion} 
\label{sec:dis}

\subsection{Generalization to Asymmetric Matrices}
\label{sec:gen}

We consider the completion of the symmetric matrices with repeated observations inspired by the application of the integration of multi-sources PMI matrices. However, we also notice that there exist many applications involving the completion of missing blocks for asymmetric matrices, for example, genomic data integration \citep{cai2016structured}, multimodality data analysis \citep{xue2020integrating} and other applications mentioned in the introduction. Hence, in this section, we introduce an algorithm designed for asymmetric matrices without repeated observations. Now assume that $\W^\ast \in \RR^{N_1 \times N_2}$ of rank $r$ is asymmetric with $\lambda_{\max} = \sigma_1(\W^\ast) > \lambda_{\min} = \sigma_r(\W^\ast) > 0$. Furthermore, we have a noisy-corrupted matrix $\W$ such that
\begin{equation*}
    \W = \W^\ast + \E\, ,
\end{equation*}
where the entries of $\E$ are independent mean-zero sub-Gaussian noise with sub-Gaussian norm bounded by $\sigma$. For the $s$th source, we sample two index sets $\calV_{s1} \subseteq [N_1]$ and $\calV_{s2} \subseteq [N_2]$ independently such that for each $i \in [N_1]$ and  $j \in [N_2]$ , we assign $i$ to $\calV_{s1}$ with probability $p_{s1}$ and $j$ to $\calV_{s2}$ with probability $p_{s2}$ independently:  
$$\PP(i \in \calV_{s1} ) = p_{s1} \in (0,1), \text{ for } i \in [N_1], \quad \PP(j \in \calV_{s2}) = p_{s2} \in (0,1), \text{ for } j \in [N_2], s \in [m] \, .$$
With the index sets $\calV_{s1}$ and $\calV_{s2}$, a matrix $\W^s$ is observed
\begin{equation}
    \W^s = \W_{\calV_{s1}, \calV_{s2} } 
    \text{ for } s \in [m] \,.
\label{eq:mod2}   
\end{equation}
Let $\calV_1^{\ast} = \cup_{s=1}^m \calV_{s1}$ and $\calV_2^{\ast} = \cup_{s=1}^m \calV_{s2}$, then our task is to recover 
$$\W_0^{\ast} \equiv \W^\ast_{ \calV_1^{\ast}, \calV_2^{\ast}} =  \big[\W^\ast(i,j)\big]_{j \in \calV_2^{\ast}}^{i\in \calV_1^{\ast}} \in \RR^{n_1 \times n_2 }, \quad \mbox{where $n_k = |\calV_k^{\ast} |, k = 1,2$ .}
$$
Without loss of generality, we assume $\calV_1^{\ast} = [n_1]$ and  $\calV_2^{\ast} = [n_2]$. The estimation procedure is summarized in Algorithm \ref{alg:OMC2}. 

\begin{algorithm}[ht!]
\SetAlgoLined
{\bf Input:} $m$ matrices $\{\W^s\}_{s \in [m]}$ and the corresponding index sets $\big \{ \calV_{s1}, \calV_{s2} \big\}_{s \in [m]}$; the rank $r$; $n_1 = |\cup_{s=1}^m \calV_{s1}|$ and $n_2 = |\cup_{s=1}^m \calV_{s2}|$\; 
{\bf Step I Aggregation:} Create $\widetilde \W \in \RR^{n_1 \times n_2}$ as follows:
\begin{equation}
    \widetilde \W(i,j) =  \W^s(v_i^{s1}, v_j^{s2}) \text{ if } i \in \calV_{s1} \text{ and } j \in \calV_{s2} \text{ for some } s \in [m]
\end{equation}
for all pairs of $(i,j)$ such that $\calS_{ij} \equiv \sum_{s=1}^m \mathbbm{1}(i \in \calV_{s1}, j \in \calV_{s2}) > 0$, where $v_i^{s1}$($v_j^{s2}$) denotes the row(column) index in $\W^s$ corresponding to the $i$th row($j$th column) of $\W_0^\ast$, and the entries in the missing blocks with $\calS_{ij} = 0$ are initialized as zero. 

{\bf Step II (a) Spectral initialization:} \For{$1\leq s \leq m$}{
  Let $\widetilde \U_{s} \widetilde \bSigma_s \widetilde \V_{s}^{\top}$ be the rank-$r$ SVD of $\W^s$. 
} 

{\bf Step II (b) Estimation of missing parts:} \For{$1\leq s < k \leq m$}{
  Obtain $\widetilde \W_{s k}$ and $\widetilde \W_{k s}$ using $\widetilde \U_{s}, \widetilde \bSigma_s, \widetilde \V_{s}, \widetilde \U_{k}, \widetilde \bSigma_k, \widetilde \V_{k}$:
  \begin{equation*}
    \widetilde \W_{s k} \equiv  \widetilde \U_{s1} \widetilde \bSigma_s^{1/2} \G( \widetilde \bSigma_s^{1/2} \widetilde \U_{s2}^{\top} \widetilde \U_{k1}  \widetilde \bSigma_k^{1/2})   \widetilde \bSigma_k^{1/2} \widetilde \V_{k2}^\top \, ,
\end{equation*}
 \begin{equation*}
    \widetilde \W_{k s} \equiv  \widetilde \U_{k2} \widetilde \bSigma_k^{1/2} \G( \widetilde \bSigma_k^{1/2} \widetilde \V_{k1}^{\top} \widetilde \V_{s2}  \widetilde \bSigma_s^{1/2})   \widetilde \bSigma_s^{1/2} \widetilde \V_{s1}^\top \, .
\end{equation*}
Here $\widetilde \U_{s} = (\widetilde \U_{s1}^{\top},\widetilde \U_{s2}^{\top})^{\top}$, $\widetilde \U_{s} = (\widetilde \V_{s1}^{\top},\widetilde \V_{s2}^{\top})^{\top}$, $\widetilde \U_{k} = (\widetilde \U_{k1}^{\top},\widetilde \U_{k2}^{\top})^{\top}$, and $\widetilde \V_{k} = (\widetilde \V_{k1}^{\top},\widetilde \U_{k2}^{\top})^{\top}$ are decomposed similarly to \eqref{eq: Wsk noise}. $\widetilde \W_{s k}$ and $\widetilde \W_{s k}$ are used like in \eqref{def: hat W}. If a missing entry $(i,j)$ is estimated by multiple pairs of sources $(s,k)$, choose the one estimated by the first pair. Denote the imputed matrix as $\widehat \W$. } 
 {\bf Step III Low rank approximation:} Obtain the rank-$r$ SVD of $\widehat \W$: $\widehat \W_r = \widehat \U \widehat \bSigma \widehat \V^\top$. 
 {\bf Output: } $\widehat \U$, $\widehat \bSigma$,  $\widehat \V$.
 \caption{BONMI for asymmetric matrices.}
\label{alg:OMC2}
\end{algorithm}
The model \eqref{eq:mod2} is more general than \eqref{eq:mod} in two ways: (i) \eqref{eq:mod2} considers the asymmetric matrix and (ii) \eqref{eq:mod2} does not assume repeated observations. To be specific, the overlapping parts of $\W^s$ are the same for different sources, which means that they are only observed once and Step I Aggregation in Algorithm \ref{alg:OMC} is not applicable now. The two relaxations make the model \eqref{eq:mod2} more flexible and realistic for the applications mentioned above. Assume that $\W_0^{\ast}$ has SVD
\begin{equation*}
    \W_0^{\ast} = \U_0^{\ast} \bSigma_0^{\ast} (\V_0^{\ast} )^{\top} = \X^{\ast} (\Y^{\ast})^{\top},
\end{equation*}
where $\U_0^\ast \in \RR^{n_1 \times r}$ are the left-singular vectors, $\V_0^\ast \in \RR^{n_2 \times r}$ are the right-singular vectors, and $\bSigma_0^\ast$ is an $r \times r$ diagonal matrix with singular values in a descending order. In addition, $\X^{\ast} = \U_0^{\ast} ( \bSigma_0^{\ast})^{1/2} \in \RR^{n_1 \times r}$ and $\Y^{\ast} = \V_0^{\ast} ( \bSigma_0^{\ast})^{1/2} \in \RR^{n_2 \times r}$. Let $\widehat \X = \widehat \U \widehat \bSigma^{1/2}$ and $\widehat \Y = \widehat \V \widehat \bSigma^{1/2}$  be the output of Algorithm \ref{alg:OMC2}. Let $N = \max\{\,N_1,N_2\,\}$ and $p_0 = \min_{s \in [m]}\{\, p_{s1},p_{s2} \, \}$. Similar results to Theorem \ref{theorem: W13} and Theorem \ref{theorem:multi-sources} can be provided. For example, when $m=2$, if $p_0 = o (1/\log N \big)$ or $p_0$ is bounded away from zero, we can prove that under similar assumptions, 
$$\|\widehat \X \O_X - \X^{\ast}\|  \lesssim \frac{ \{(1-p_0) K^2   +1 \} K}{\sqrt{\lambda_{\min}}}  \sqrt{N} \sigma \, ,$$

$$\|\widehat \Y \O_Y - \Y^{\ast}\|  \lesssim \frac{ \{(1-p_0) K^2   +1 \} K}{\sqrt{\lambda_{\min}}}  \sqrt{N} \sigma\, , $$
and otherwise, 
$$
       \|\widehat \X \O_X - \X^{\ast}\| \lesssim \frac{ \{(1-p_0) K^2 (p_0 \log N) +1 \} K}{\sqrt{\lambda_{\min}}}  \sqrt{N} \sigma  \, , 
$$ 
$$
       \|\widehat \Y \O_Y - \Y^{\ast}\| \lesssim \frac{ \{(1-p_0) K^2 (p_0 \log N) +1 \} K}{\sqrt{\lambda_{\min}}}  \sqrt{N} \sigma \, ,
$$ 
for some rotation matrices $\O_X,\O_Y \in \scrO^{r \times r}$. The proofs are the simple extensions of the proof of  Theorem \ref{theorem: W13} and Theorem \ref{theorem:multi-sources}.  


\subsection{Conclusion}
This paper proposes BONMI, which aims at the matrix completion under the block-wise missing pattern. Our method is computationally efficient and attains near-optimal error bound. The performance of our algorithm is verified by simulation and real data analysis. For theoretical guarantee, we require the sampling probability of each source has the order of $\sqrt{\log N/N}$, which is a small order of $N$ and can be satisfied easily in many applications. Besides, we extend BONMI to asymmetric matrices without repeated observations for other potential applications such as genomic data integration. 




\normalem
\bibliographystyle{apalike}
\bibliography{ref}

\newpage
\begin{center}
\textit{\large Technical supplementary material to}
\end{center}
\begin{center}
\title{\LARGE Multi-source Learning via Completion of Block-wise Overlapping Noisy Matrices}
\vskip10pt
\author{Doudou Zhou, Tianxi Cai and Junwei Lu\\}
\end{center}

\setcounter{section}{0}
\renewcommand{\thesection}{S.\arabic{section}}
\setcounter{equation}{0}
\counterwithout{equation}{section}
\renewcommand{\theequation}{S.\arabic{equation}}
\setcounter{theorem}{0}
\counterwithout{theorem}{section}
\renewcommand{\thetheorem}{S.\arabic{theorem}}
\renewcommand{\thelemma}{S.\arabic{theorem}}

This document contains the supplementary material to the paper ``Multi-source Learning via Completion of Block-wise Overlapping Noisy Matrices". We mainly provide technical details of proving Proposition \ref{prop: exactly recover}, Theorems \ref{theorem: W13} and \ref{theorem:multi-sources} here.


\section{Proof of Proposition \ref{prop: exactly recover}}

\begin{proof}
First, we prove ${\rm rank}(\W^\ast\subsksk)=r$. Let $n_{sk} \equiv |\calV_s \cap \calV_k|$, then by Lemma \ref{prop: ns}, we have 
\begin{equation}
 n_{s k} \geq \frac{3 p_s p_k N}{2} \geq C \mu_0 r \log N
\end{equation}
holds with probability $1- O(1/N^3)$ for some sufficiently large constant $C$, where the second inequality comes from Assumption \ref{assump: prob}. Then, by Lemma \ref{lemma:lower bound U}, we have
$$
\lambda_{r}(\W^\ast\subsksk) 
\geq \sigma_{\min}(\U^{\ast}_{s\cap k}) \lambda_{r}(\bSigma^{\ast})  \sigma_{\min}(\U^{\ast}_{s\cap k} ) \geq \frac{n_{sk} \lambda_{\min}}{2 N}  > 0
$$
with probability $1-1/N^3$. As a result, ${\rm rank}(\W^\ast\subsksk)=r$ with probability as least $1 - O(1/N^3)$. Under the event, since $\W^\ast\subsksk$ is a principal sub-matrix of $\W_s^{\ast}$, we have ${\rm rank}(\W_s^{\ast}) \geq {\rm rank}(\W^\ast\subsksk)$. Besides, $\W_s^{\ast}$ is a principal sub-matrix of $\W^{\ast}$, we have ${\rm rank}(\W_s^{\ast}) \leq {\rm rank}(\W^{\ast}) = r$. Combining the two inequalities, we have ${\rm rank}(\W_s^{\ast}) = r$. The same conclusion holds for $\W_k^{\ast}$. We then prove that  $\W^\ast\subsnkkns$ has the representation of \eqref{eq: Wsk}. Recall the eigen-decomposition of $\W^{\ast} = \U^{\ast} \bSigma^{\ast} (\U^{\ast} )^{\top}$ and by definition we will have 
\begin{equation}
    \W^\ast\subsksk  =  \U^{\ast}_{s\cap k} \bSigma^{\ast} (\U^{\ast}_{s\cap k} )^{\top} = \V_{s2}^{\ast} \bSigma_s^{\ast} (\V_{s2}^{\ast})^\top = \V_{k1}^{\ast} \bSigma_k^{\ast} (\V_{k1}^{\ast})^\top,
\label{eq:A1}    
\end{equation}
which also implies that ${\rm rank}(\V_{s2}^{\ast}) = {\rm rank}( \V_{k1}^{\ast}) = r$.  Multiplying  $\V_{k1}^{\ast}$ on the both sides of the last equation, we obtain 
$$\V_{s2}^{\ast} (\bSigma_s^{\ast})^{1/2} (\bSigma_s^{\ast})^{1/2}  (\V_{s2}^{\ast})^\top \V_{k1}^{\ast} = \V_{k1}^{\ast} (\bSigma_k^{\ast})^{1/2} (\bSigma_k^{\ast})^{1/2}  (\V_{k1}^{\ast})^\top \V_{k1}^{\ast}$$
and the following equation
$$\V_{k1}^{\ast} (\bSigma_k^{\ast})^{1/2} = \V_{s2}^{\ast} (\bSigma_s^{\ast})^{1/2} \widehat \R,$$ where $\widehat \R = (\bSigma_s^{\ast})^{1/2} (\V_{s2}^{\ast})^\top \V_{k1}^{\ast} \big( (\V_{k1}^{\ast})^\top \V_{k1}^{\ast}\big)^{-1}  (\bSigma_k^{\ast})^{-1/2}.$ It is easy to verify that $\widehat \R^{\top} \widehat \R = \mathbf{I}_r$ and then it is obvious that $$ \widehat \R = \arg \min_{\R \in \scrO^{r \times r}} \|\V_{s2}^{\ast} (\bSigma_s^{\ast})^{1/2} \R - \V_{k1}^{\ast} (\bSigma_k^{\ast})^{1/2} \|_{\rm F}.$$
Then by  Lemma 22 of \cite{ma2018implicit}, we prove that $\widehat \R = \G\big((\V_{s2}^{\ast} (\bSigma_s^{\ast})^{1/2})^\top \V_{k1}^{\ast} (\bSigma_k^{\ast})^{1/2} \big)$. 

Again by \eqref{eq:A1}, we have
\begin{equation}
     (\U^{\ast}_{s\cap k} )^{\top} = \bSigma^{^{\ast}-1} \big( (\U^{\ast}_{s\cap k} )^{\top} \U^{\ast}_{s\cap k} \big)^{-1} (\U^{\ast}_{s\cap k} )^{\top} \V_{k1}^{\ast} \bSigma_k^{\ast} (\V_{k1}^{\ast})^\top.
    \label{eq:A2}
\end{equation}
 In addition, we have
 \begin{equation}
 \begin{aligned}
     \U_{s\cap k}^{\ast} \bSigma^{\ast} (\U_{s\backslash k})^\top & =  \W^\ast\subskkns = \V_{k1}^{\ast} \bSigma_k^{\ast} (\V_{k2}^{\ast})^\top
 \end{aligned}
\label{eq:A3}
 \end{equation}

Combining \eqref{eq:A2} and \eqref{eq:A3}, we have
\begin{equation}
\begin{aligned}
         & (\U^{\ast}_{s\cap k} )^{\top} \V_{k1}^{\ast} \{(\V_{k1}^{\ast})^\top \V_{k1}^{\ast}\}^{-1} (\V_{k2}^{\ast})^\top \\
         & = \bSigma^{^{\ast}-1} \big( (\U^{\ast}_{s\cap k} )^{\top} \U^{\ast}_{s\cap k} \big)^{-1} (\U^{\ast}_{s\cap k} )^{\top} \V_{k1}^{\ast} \bSigma_k^{\ast} (\V_{k1}^{\ast})^\top \V_{k1}^{\ast} \{(\V_{k1}^{\ast})^\top \V_{k1}^{\ast}\}^{-1} (\V_{k2}^{\ast})^\top \\
         & = \bSigma^{^{\ast}-1} \big( (\U^{\ast}_{s\cap k} )^{\top} \U^{\ast}_{s\cap k} \big)^{-1} (\U^{\ast}_{s\cap k} )^{\top} \V_{k1}^{\ast} \bSigma_k^{\ast} (\V_{k2}^{\ast})^\top
         \\ 
         & =  \bSigma^{^{\ast}-1} \big( (\U^{\ast}_{s\cap k} )^{\top} \U^{\ast}_{s\cap k} \big)^{-1} (\U^{\ast}_{s\cap k} )^{\top}  \U_{s\cap k}^{\ast} \bSigma^{\ast} (\U_{s\backslash k})^\top  = (\U_{s\backslash k})^\top,
\label{eq:A4}
\end{aligned}
\end{equation}
where the first equation comes from \eqref{eq:A2} and the second equation comes from \eqref{eq:A3}. 

Finally, 

\begin{equation*}
\begin{aligned}
& \V_{s1}^{\ast} ( \bSigma_s^{\ast})^{1/2} \G( ( \bSigma_s^{\ast})^{1/2} (\V_{s2}^{\ast})^{\top} \V_{k1}^{\ast} (\bSigma_k^{\ast})^{1/2})  ( \bSigma_k^{\ast})^{1/2} (\V_{k2}^{\ast})^{\top} \\
& = \V_{s1}^{\ast} ( \bSigma_s^{\ast})^{1/2}  (\bSigma_s^{\ast})^{1/2} (\V_{s2}^{\ast})^\top \V_{k1}^{\ast} \{(\V_{k1}^{\ast})^\top \V_{k1}^{\ast}\}^{-1}  (\bSigma_k^{\ast})^{-1/2}  ( \bSigma_k^{\ast})^{1/2} (\V_{k2}^{\ast})^{\top} \\
& = \U^{\ast}_{s\backslash k} \bSigma^{\ast} (\U^{\ast}_{s \cap k })^\top \V_{k1}^{\ast} \{(\V_{k1}^{\ast})^\top \V_{k1}^{\ast}\}^{-1}   (\V_{k2}^{\ast})^{\top} \\
& = \U^{\ast}_{s\backslash k} \bSigma^{\ast} (\U^{\ast}_{k \backslash s})^\top = \W^{\ast}_{s\backslash k,k \backslash s},
\end{aligned}
\end{equation*}
where the first equation comes from $\widehat \R = \G\big((\V_{s2}^{\ast} (\bSigma_s^{\ast})^{1/2})^\top \V_{k1}^{\ast} (\bSigma_k^{\ast})^{1/2} \big)$, the second equation comes from $\V_{s1}^{\ast}  \bSigma_s^{\ast} (\V_{s2}^{\ast})^\top = \W^{\ast}_{s\backslash k,s \cap k } = \U^{\ast}_{s\backslash k} \bSigma^{\ast} (\U^{\ast}_{s \cap k })^\top$, and the third equation comes from \eqref{eq:A4}.
Then we finish the proof. 
\end{proof}

\section{Proof of Theorem \ref{theorem: W13}}
\label{proof:W13}

When $m=2$, we adopt the notations of Section \ref{method} by assuming the two observed sub-matrices are $\W^s$ and $\W^k$. To prove the theorem, recall that $\widetilde \W_{s k}$ defined by \eqref{eq: Wsk noise} is the estimate of $\W^\ast\subsnkkns$, the main effort lies on the perturbation bound of $\| \widetilde \W_{s k} - \W^\ast\subsnkkns \|$. After we obtain it, the perturbation bound of $\| \widetilde \W - \W^{\ast} \|$ can also be figured out. As a result, the error of the rank $r$ factorization of $\widetilde \W$ can also be bounded, which leads to Theorem \ref{theorem: W13}. Before we derive $\| \widetilde \W_{s k} - \W^\ast\subsnkkns \|$, we need the basic spectral properties of $\W_0^{\ast}$ defined in \eqref{eq: Wstar},  $\W_s^{\ast}$, $\W_k^{\ast}$ defined in \eqref{eq: Ws Wk}, which are presented in the Section \ref{sec:property}. 

\subsection{Characterization of The Underlying Matrix}
\label{sec:property}
Let $n_s \equiv |\calV_s|$, $n_k \equiv |\calV_k|$ and $n_{sk} \equiv |\calV_s \cap \calV_k|$. First, by Lemma \ref{prop: ns}, we have 
\begin{equation}
    \frac{p_l N}{2} \leq n_l \leq \frac{3 p_l N}{2}, l=s,k \ {\rm and} \ \frac{p_s p_k N}{2} \leq n_{s k} \leq \frac{3 p_s p_k N}{2}
\label{cond1}    
\end{equation}
hold simultaneously with probability $1- O(1/N^3)$. Throughout, our analysis is conditional on \eqref{cond1}. By Proposition \ref{sigma_Ws}, we have  
 $$\lambda_r (\W^{\ast}_l) \geq \frac{n_l \lambda_{\min}}{2 N}\geq \frac{p_l \lambda_{\min}}{4}, l = s,k$$ 
 hold simultaneously with  probability $1-2/N^3$ since by the Assumption \ref{assump: prob}, we have $n_l \geq N p_0/2 \geq 16 \mu_0 r (\log r  + \log N^3), l=s, k$. Then by Lemma \ref{lemma: Incoherence condition of principal submatrix}, we will have $$\mu_l := \mu(\V_l^{\ast}) = \frac{n_l}{r} \max_{i=1,\dots,n_l} \sum_{j=1}^r \V_l^{\ast}(i,j)^2 \leq 2 \tau \mu_0, l = s, k.$$
 In addition, by Proposition \ref{prop:sigma1}, we have 
$$\lambda_1(\W_l^{\ast}) \leq \frac{n_l r \mu_0}{N} \lambda_{\max} \leq \frac{3 p_l r \mu_0}{2} \lambda_{\max}, l =s, k.$$ 
As a result, we have the condition number of $\W_l^{\ast}$:
\begin{equation}
    \tau_l := \lambda_1(\W_l^{\ast})/\lambda_r(\W_l^{\ast}) \leq 6 r \mu_0 \tau, l = s,k.
\end{equation}

\subsection{Imputation Error}
\label{sec: Imputation Error}
After we characterize the spectral properties of $\W_0^{\ast}$ defined in \eqref{eq: Wstar},  $\W_l^{\ast}, l =s,k$, we begin to control  $\| \widetilde \W_{s k} - \W^\ast\subsnkkns  \|$. Using the notations of Proposition \ref{prop: exactly recover} and Section \ref{method}, we define 
\begin{equation}
\begin{aligned}
& \A = \V_s^{\ast} (\bSigma_s^{\ast})^{1/2}; \ \B = \V_k^{\ast} (\bSigma_k^{\ast})^{1/2}; \ \widetilde \A = \widetilde \V_s (\widetilde \bSigma_s)^{1/2}; \ \widetilde \B = \widetilde \V_k (\widetilde \bSigma_k)^{1/2}; \\
& \A_1 = \V_{11}^{\ast} (\bSigma_1^{\ast})^{1/2}; \ \A_2 = \V_{12}^{\ast} (\bSigma_1^{\ast})^{1/2}; \ \B_1 = \V_{21}^{\ast} (\bSigma_2^{\ast})^{1/2}; \ \B_2 = \V_{22}^{\ast} (\bSigma_1^{\ast})^{1/2}; \\
& \widetilde \A_1 = \widetilde \V_{11} (\widetilde \bSigma_1)^{1/2}; \ \widetilde \A_2 = \widetilde \V_{12} (\widetilde \bSigma_1)^{1/2}; \ \widetilde \B_1 = \widetilde \V_{21} (\widetilde \bSigma_2)^{1/2}; \ \widetilde \B_2 = \widetilde \V_{22} (\widetilde \bSigma_2)^{1/2}
\end{aligned}
\end{equation}
and $\Q_A = \G(\widetilde \A^\top \A)$, $\Q_B = \G(\widetilde \B^\top \B)$, $\widetilde \O = \G(\widetilde \A_2^\top \widetilde \B_1)$. It is easy to see that 
\begin{equation}
\widetilde \W_{s k} =  \widetilde \A_1 \widetilde \O^\top \widetilde \B_2^\top = \widetilde \A_1 \Q_A (\Q_A^\top \widetilde \O^\top \Q_B) \Q_B^\top \widetilde \B_2^\top = \widetilde \A_1 \Q_A \G (\Q_B^\top \widetilde \B_1^\top \widetilde \A_2 \Q_A) \Q_B^\top \widetilde \B_2^\top.  
\label{def:A,B}
\end{equation}

Then by Proposition \ref{prop: exactly recover},  we have 
\begin{equation}
\begin{aligned}
 &  \|\widetilde \W_{s k} - \W^\ast\subsnkkns \|   = 
    \|(\widetilde \A_1 \Q_A) (\Q_A^\top \widetilde \O^\top \Q_B) (\Q_B^\top \widetilde \B_2^\top) - \A_1 \O^\top \B_2^\top \| \\ 
    & = \| (\widetilde \A_1 \Q_A) (\Q_A^\top \widetilde \O^\top \Q_B) (\Q_B^\top \widetilde \B_2^\top) - \A_1 (\Q_A^\top \widetilde \O^\top \Q_B) (\Q_B^\top \widetilde \B_2^\top)\\ & + \A_1 (\Q_A^\top \widetilde \O^\top \Q_B) (\Q_B^\top \widetilde \B_2^\top) - \A_1 (\Q_A^\top \widetilde \O^\top \Q_B)  \B_2^\top \\& + \A_1 (\Q_A^\top \widetilde \O^\top \Q_B)\B_2^\top - \A_1 \O^\top \B_2^\top\| \\
    & \leq \|\widetilde \B_2\| \| \widetilde \A_1 \Q_A - \A_1\|  + \|\widetilde \A_1\| \| \widetilde \B_2 \Q_B - \B_2\| + \|\A_1\| \|\B_2\| \| \Q_A^\top \widetilde \O^\top \Q_B - \O\|. 
\end{aligned}    
\end{equation}
Applying Proposition \ref{prop: Probability Bound}, Lemma \ref{lemma: tilde Ws}, Lemma \ref{lemma:A,B},  Lemma \ref{lemma: orthogonal}, with $f(p_0, N)$ defined in \eqref{def: f(p0,N)}, we have 
\begin{equation}
\begin{aligned}
 &   \|\widetilde \W_{s k} - \W^\ast\subsnkkns \| \\
    &\lesssim (1-p_0)(\|\B\| \| \widetilde \A \Q_A - \A\| + \|\A\| \| \widetilde \B \Q_B - \B\| + \|\A\| \|\B\| \| \Q_A^\top \widetilde \O^\top \Q_B - \O\|) \\
    & \lesssim  (1-p_0) \left \{ r \mu_0 \tau +  f(p_0, N)^2 (r \mu_0 \tau)^2 \right \} (\|\widetilde \E_1\| + \|\widetilde \E_2\|) \\
    & \lesssim (1-p_0)  (r \mu_0 \tau)^2  f(p_0, N)^2 \sqrt{N p_0} \sigma
\end{aligned}    
\end{equation}
with probability $1-20/N^3 = 1 - O(1/N^3)$.

\subsection{Completion Error}
\label{sec: completion error}
After we impute the missing blocks, we can bound $\| \widehat \W - \W^{\ast} \|$ where $\widehat \W$ is defined as \eqref{def: hat W}. Notice that 
\begin{equation}
    \widehat \W = \W^{\ast} +   \widetilde \E + \widetilde \F,
\label{eq: dec hatW}    
\end{equation}
where $$\widetilde \E = \begin{bmatrix}
\E\subsnksnk^s & \E\subsnksk^s & \O \\
\E\subsksnk^s &\alpha_s \E\subsksk^s + \alpha_k \E\subsksk^k & \E\subskkns^k\\
\O & \E\subknssk^k  &  \E\subknskns^k   \\
\end{bmatrix} ,$$ 
 and $$\widetilde \F = \begin{bmatrix}
\O & \O & \widetilde \W_{s k} - \W^\ast\subsnkkns \\
\O &  \O & \O\\
\widetilde \W^{\top}_{s k} - \W^\ast\subknssnk & \O &  \O \\
\end{bmatrix}.$$ Then we only need to bound $\|\widetilde \E\|$ and $\| \widetilde \F \|$. It is easy to see that $\| \widetilde \F \| = \|\widetilde \W_{s k} - \W^\ast\subsnkkns\|$, then we only need to bound $\|\widetilde \E\|$.  First, by Corollary $3.3$ of  \cite{bandeira2016sharp}, we have 
$$\EE \| \widetilde \E \| \lesssim \sigma^{\ast} + \sigma \sqrt{\log n},$$
where $\sigma = \max \{\sigma_s,\sigma_k\}$ and $\sigma^{\ast} = \max_i \sqrt{\sum_j \EE \widetilde \E_{ij}^2 }$. It is easy to see that 
$$\sigma^{\ast}  =  \max \{\sqrt{n_s} \sigma_s,\sqrt{n_k} \sigma_k,\sqrt{ (n_s - n_{sk})\sigma_s^2 + (n_k - n_{sk})\sigma_k^2 + n_{sk}(\alpha_s^2 \sigma_s^2 + \alpha_k^2 \sigma_k^2)} \}.$$
In addition, by Lemma 11 and Proposition 1 of \cite{chen2015fast}, there exists a universal constant $c > 0$ such that
$$\PP \{ \| \widetilde \E \| \geq c (\sigma^{\ast} + \sigma \log n) \} \leq N^{-12}.$$
In order to minimize $\| \widetilde \E \|$ with regard to $\alpha_s$ and $\alpha_k$, the best we can do is to minimize its upper bound. It is easy to see that 
$$(\alpha_1^{\ast},\alpha_2^{\ast}) = (\sigma_2^2/(\sigma_1^2+\sigma_2^2),\sigma_1^2/(\sigma_1^2+\sigma_2^2)) = \arg \min_{\alpha_1 + \alpha_2 = 1,\alpha_1 >0, \alpha_2 >0} \alpha_1^2 \sigma_1^2 + \alpha_2^2 \sigma_2^2.$$
In reality, we don't know $\sigma_s$ and $\sigma_k$, but we can estimate them by \eqref{def: sigmas}. 

Since 
$$\alpha_1^2 \sigma_1^2 + \alpha_2^2 \sigma_2^2 \leq (\alpha_1^2 + \alpha_2^2) \sigma^2 \leq (\alpha_1+ \alpha_2)^2 \sigma^2 = \sigma^2,$$
we have $\sigma^{\ast} \leq \sqrt{n} \sigma$. So $\| \widetilde \E \| \lesssim \sigma^{\ast} \leq \sqrt{n} \sigma$ with probability at least $1-n^{-12} \geq 1 - O(1/N^3)$. By $n = n_s + n_k - n_{sk} \leq 3N p_s/2 + 3N p_k/2 - N p_s p_k/2 \lesssim N p_0$, we get $\sigma^{\ast} \lesssim \sqrt{N p_0} \sigma$. Finally, we have

\begin{equation}
     \| \widehat \W - \W^{\ast} \| \leq \| \widetilde \E \| + \| \widetilde \F \| \lesssim \sqrt{N p_0} \sigma +  (1-p_0)  (r \mu_0 \tau)^2 f(p_0, N)^2 \sqrt{N p_0} \sigma. 
\end{equation}

\subsection{Low-rank Approximation}

The last step is to do rank-$r$ eigendecomposition on $\widehat \W$ to obtain $\widehat \W_r = \widehat \U \widehat \bSigma \widehat \U^\top = \widehat \X \widehat \X^\top$ where $\widehat \X = \widehat \U \widehat \bSigma^{1/2}$. Then there exists an orthogonal matrix $\O_X$ such that
\begin{equation}
\begin{aligned}
    \|\widehat \X \O_X - \X^{\ast}\| & \lesssim \frac{\| \widehat \W - \W^{\ast} \| r \mu_0  \tau}{\sqrt{\lambda_r(\W_0^{\ast})}}   \lesssim \frac{\| \widehat \W - \W^{\ast} \|  r \mu_0  \tau}{\sqrt{\lambda_{\min} p_0}}\\
    & \lesssim \{ (1-p_0)  (r \mu_0 \tau)^2 f(p_0, N)^2 +1 \} r \mu_0  \tau \sqrt{\frac{N}{\lambda_{\min}}} \sigma .
\end{aligned}
\end{equation}
by a similar proof as Lemma \ref{lemma:A,B} and the fact that $\lambda_r(\W_0^{\ast}) \geq \lambda_r(\W^{\ast}_s) \geq p_0 \lambda_{\min}/4$. Finally, this upper bound holds with probability at least $1-O(1/N^3)$ by the probability union bound.

\section{Proof of Theorem \ref{theorem:multi-sources}}

We know that $n \sim {\rm Binomial}(N,1- \prod_{s=1}^m(1-p_s))$, so by the same argument to Lemma \ref{prop: ns}, we have
$$N \{1- \prod_{s=1}^m(1-p_s)\}/2 \leq n \leq 3 N \{1- \prod_{s=1}^m(1-p_s)\}/2 $$
with probability $1 - O(1/N^3)$. As a result,
\begin{equation}
\{1- (1-p_0)^m\} \lambda_{\min} \lesssim  \lambda_r(\W_0^{\ast}) \leq \lambda_1(\W_0^{\ast}) \lesssim  \{1- (1-p_0)^m\} r \mu_0 \lambda_{\max}
\label{eq:n}
\end{equation}
by a similar argument as in the proof of Theorem \ref{theorem: W13} and the Assumption \ref{assump: prob} that $p_s/p_0 = O(1)$. In addition, let $\E = \widehat \W - \W_0^{\ast}$, then by a similar decomposition as in \eqref{eq: dec hatW}, we will have 
$$\| \E \| \leq  \| \widetilde \E \|  + \sum_{s=1}^{m-1} \sum_{k=s+1}^{m} \| \T^{sk} \circ (\widetilde \W_{sk} - \W^{\ast}_{ s \backslash k, k \backslash s } ) \|$$
where $\widetilde \E \in R^{n \times n}$ with
$$\widetilde \E(i,j) = \sum_{s=1}^m \alpha^s_{ij} \E^s(v_i^s,v_j^s)  \mathbbm{1}(i,j \in \calV_s), \text{ for } \calS_{ij} > 0$$
and $\widetilde \E(i,j) = 0$ for $\calS_{ij} =0$. Here we denote $\circ$ as the Hadamard product operator and $\T^{sk},s \neq k \in [m]$ are $0/1$ matrices decided by the Algorithm \ref{alg:OMC}. According to the Algorithm \ref{alg:OMC}, the nonzero entries of  $\T^{sk},s \neq k \in [m]$ are block-wise, which implies that 
$$ \| \T^{sk} \circ (\widetilde \W_{sk} -  \W^\ast\subsnkkns ) \| \leq \|\widetilde \W_{sk} -  \W^\ast\subsnkkns \|.$$
Then, by the proof of Theorem \ref{theorem: W13}, we have
$\| \widetilde \E \| \lesssim \sqrt{N p_0} \sigma$
and 
$$\|\widetilde \W_{sk} - \W^{\ast}_{ s \backslash k, k \backslash s } \| \lesssim (1-p_0)  (r \mu_0 \tau)^2 f(p_0, N)^2 \sqrt{N p_0} \sigma $$
hold simultaneously with probability $1-O(m^2/N^3)$ for $1 \leq s < k \leq m$. As a result,
\begin{equation}
   \|\widehat \W  - \W_0^{\ast} \| \lesssim m (m-1) (1-p_0)  (r \mu_0 \tau)^2 f(p_0, N)^2 \sqrt{N p_0} \sigma + \sqrt{N p_0} \sigma 
\end{equation} 
and
\begin{equation}
    \|\widehat \X \O_X - \X^{\ast}\|\lesssim \frac{\|  \widehat \W  - \W_0^{\ast} \| r \mu_0  \tau}{\sqrt{\lambda_r(\W_0^{\ast})}} .
\end{equation} 
By \eqref{eq:n}, we have
\begin{equation}
    \|\widehat \X \O_X - \X^{\ast}\|\lesssim \{ 1 +  m^2 (1-p_0)  (r \mu_0 \tau)^2 f(p_0, N)^2 \sqrt{\frac{p_0}{1- (1-p_0)^m}} \} r \mu_0  \tau \sqrt{\frac{N}{\lambda_{\min}} } \sigma
\end{equation} 
with probability  $1-O(m^2/N^3)$.  Given $0 < \epsilon <1$, we have 
$$\PP(n < (1-\epsilon)N) = O( \frac{1}{N^3})$$
when $m \approx \log(\epsilon - \sqrt{\frac{3 \log N}{2 N}})/\log(1-p_0)$ by the fact that $n \sim  {\rm Binomial}(N,1- \prod_{s=1}^m(1-p_s))$ and the Bernstein inequality. Since $\lim_{N \to \infty} \sqrt{\log N/ N} = 0$ we have $m \approx \log \epsilon/\log(1-p_0)$. Finally, we have
\begin{equation}
    \|\widehat \X \O_X - \X^{\ast}\|\lesssim \{ 1 + \frac{\log^2 \epsilon}{\log^2(1-p_0)}(1-p_0)  (r \mu_0 \tau)^2 f(p_0, N)^2 \sqrt{\frac{p_0}{1- (1-p_0)^m}} \} r \mu_0  \tau \sqrt{\frac{N}{\lambda_{\min}} } \sigma
\end{equation} 
with probability  $1-O(m^2/N^3)$.

\section{Details of the Proof of Theorem \ref{theorem: W13}}

Here we present some key lemmas and propositions needed for our proof of Theorem \ref{theorem: W13}. 

\begin{lemma}[The dimension of sub-matrix] Under the assumption that 
$$p_s \geq p_0 \geq  C \sqrt{\mu_0 r \tau \log N /N},$$
for some sufficiently large constant $C$, we have 
\begin{equation}
\label{eq: union bound}
    \frac{p_s N}{2} \leq n_s \leq \frac{3 p_s N}{2} \ {\rm and} \ \frac{p_s p_k N}{2} \leq n_{s k} \leq \frac{3 p_s p_k N}{2},  s \neq k, s,k \in [m]
\end{equation}
with probabilities $1 - O(m^2/N^3)$. 
\label{prop: ns}
\end{lemma}
\begin{proof}
By the Bernstein inequality, we have
$$\PP \{ Y \leq pn-t \} \leq \exp\{ - \frac{\frac{1}{2}t^2}{np(1-p) + \frac{1}{3}t} \} \ {\rm and} \ \PP \{ Y \geq pn+t \} \leq \exp\{ - \frac{\frac{1}{2}t^2}{ np(1-p) + \frac{1}{3}t} \}$$
if $Y \sim {\rm Binomial}(n,p)$. Since $n_s \sim  {\rm Binomial}(N,p_s)$ and $n_{sk} \sim  {\rm Binomial}(N,p_s p_k)$, let $t=\frac{p_s}{2}$, we have
$$\PP \{\frac{p_s N}{2} \leq n_s \leq \frac{3 p_s N}{2}  \} \geq 1 - 2 \exp \{ - \frac{3 p_s N}{28}  \}.$$
Similarly, we have
$$\PP \{\frac{p_s p_k N}{2} \leq n_{sk} \leq \frac{3 p_s p_k N}{2}  \} \geq 1 - 2 \exp \{ - \frac{3 p_s p_k N}{28}  \}.$$
In addition, by $p_s \geq p_0 \geq C \sqrt{\mu_0 r \tau \log N /N}$, we have $\exp \{ - 3 p_s N/28\} = O(1/N^3)$ and $\exp \{ - 3 p_s p_k N/28\} = O(1/N^3)$. Finally, by the probability union bound, \eqref{eq: union bound} holds with probability $1 - O(m^2/N^3)$. 
\end{proof}

\begin{lemma}[Lemma 5, \cite{cai2016structured}] Suppose $\U \in \RR^{N \times r}$ ($N \geq r$) is a fixed matrix with orthonormal columns. Denote $\mu = \max_{1 \leq i \leq N} \frac{N}{r} \sum_{j=1}^r u_{ij}^2$.  Suppose we uniform randomly draw $n$ rows (with or without replacement) from $\U$ and denote it as $\U_{\Omega}$, where $\Omega$ is the index set. When $n \geq 4 \mu r (\log r  + c)/(1-\alpha)^2$ for some $0 < \alpha < 1$ and $c>1$, we have 
\begin{equation}
    \sigma_{\min}(\U_{\Omega}) \geq \sqrt{\frac{\alpha n}{N}}
\end{equation}
with probability $1-2 e^{-c}$.
\label{lemma:lower bound U}
\end{lemma}

By Lemma \ref{lemma:lower bound U}, we will directly have the following proposition.
\begin{proposition}
\label{sigma_Ws}
Let $\alpha = \frac{1}{2}$ and $c=\log  2 N^3$ in Lemma \ref{lemma:lower bound U}, then when $$n_s \geq 16 \mu_0 r (\log r  + \log  2 N^3),$$
we have $\sigma_{\min}(\U^{\ast}_{\calV_s}) \geq \sqrt{\frac{n_s}{2 N}}$ with probability $1-1/N^3$. In addition,  under the event, we have 
$$\lambda_r (\W^{\ast}_s) = \lambda_r (\U^{\ast}_{\calV_s} \bSigma^{\ast} (\U^{\ast}_{ \calV_s})^\top ) \geq \sigma_{\min}(\U^{\ast}_{ \calV_s}) \lambda_r(\bSigma^{\ast}) \sigma_{\min}(\U^{\ast}_{ \calV_s}) \geq \frac{n_s \lambda_{\min}}{2 N}.$$
\end{proposition}

\begin{lemma} [Incoherence condition of principal submatrix] Recall that  $\V_s^{\ast} \bSigma_s^{\ast} (\V_s^{\ast})^\top$ is the  rank-$r$ eigendecomposition of $\W_s^{\ast}$. Assume that $\lambda_r (\W^{\ast}_s) \geq \frac{n_s \lambda_{\min}}{2 N}$. Then the incoherence of $\V_s^{\ast}$ satisfies 
$$\mu_s  \equiv \mu(\V_s^{\ast}) = \frac{n_s}{r} \max_{i=1,\dots,n_s} \sum_{j=1}^r \V_s^{\ast}(i,j)^2 \leq 2 \tau \mu_0.$$
\label{lemma: Incoherence condition of principal submatrix}
\end{lemma}
\begin{proof}
Since $\W_s^{\ast} = \U^{\ast}_{ \calV_s} \bSigma^{\ast} (\U^{\ast}_{\calV_s})^\top = \V_s^{\ast} \bSigma_s^{\ast} (\V_s^{\ast})^\top,$ we have 
$$\V_s^{\ast} = \U^{\ast}_{\calV_s} (\bSigma^{\ast})^{\frac{1}{2}} \O_s^{\top} (\bSigma_s^{\ast})^{- \frac{1}{2}}$$
where $\O_s = (\bSigma_s^{\ast})^{-\frac{1}{2}} (\V_s^{\ast})^{\top} \U^{\ast}_{\calV_s} (\bSigma^{\ast})^{\frac{1}{2}} \in \scrO^{r \times r}$. Then 
$$\sum_{j=1}^r \V_s^{\ast}(i,j)^2 \leq \sum_{j=1}^r \U^{\ast}_{\calV_s}(i,j)^2 \|(\bSigma_s^{\ast})^{-\frac{1}{2}}\|^2 \|(\bSigma^{\ast})^{\frac{1}{2}}\|^2 \leq \frac{r \mu_0}{N} \frac{\lambda_{\max}}{\lambda_r(\W^{\ast}_s)}$$
As a result,
$$\mu_s = \frac{n_s}{r} \max_{i=1,\dots,n_s} \sum_{j=1}^r \V_s^{\ast}(i,j)^2 \leq \frac{n_s \mu_0}{N} \frac{\sigma_{\max}}{\sigma_r(\W^{\ast}_s)} \leq \frac{2 \lambda_{\max}}{\lambda_{\min}} \mu_0 = 2 \tau \mu_0.$$
\end{proof}

\begin{proposition} [Upper bound of the operator of the submatrix.] We have
$$\lambda_1(\W_s^{\ast}) \leq \min\{1,\frac{n_s r \mu_0}{N} \} \lambda_{\max}. $$
\label{prop:sigma1}
\end{proposition}
\begin{proof}
It is obviously that $\lambda_1(\W_s^{\ast}) = \lambda_1( \U^{\ast}_{ \calV_s} \bSigma^{\ast} (\U^{\ast}_{\calV_s})^\top ) \leq \sigma_{\max}( \U^{\ast}_{\calV_s})^2 \lambda_{\max}( \bSigma^{\ast}) \leq  \lambda_{\max}$ because $\sigma_{\max}( \U^{\ast}_{\calV_s}) \leq 1$. Besides, we have $\|\U^{\ast}_{ \calV_s}\|^2 \leq n_s \| \U^{\ast}_{ \calV_s} \|_{2,\infty}^2 \leq  n_s r \mu_0/N$ where the first inequality comes from the property of $\ell_2/\ell_{\infty}$ norm and the second inequality comes from $\mu_0 = \mu(\U^{\ast})$ and the definition of incoherence. 
\end{proof}

\subsection{Error Matrix}
Recalling that $\widetilde \W_s \equiv \widetilde \W_{\calV_s, \calV_s}$, we characterize the operator norm of $\widetilde \W_s- \W_s^{\ast}, s \in [m]$ in the Lemma \ref{lemma: tilde Ws}.

\begin{lemma}
Let $\widetilde \E_s := \widetilde \W_s - \W_s^{\ast}, s \in [m]$. Under Assumptions \ref{assump: prob}, \ref{assump: signal to noise ratio},  and the condition $p_s N/2 \leq n_s \leq 3 p_s N/2, s \in [m]$, we have
$$\|\widetilde \E_s \| \lesssim \sqrt{N p_0} \sigma   \ll  \frac{p_0 \lambda_{\min}}{4} \leq \lambda_r(\W_s^{\ast}), s \in [m]$$
with probability $1-O(m/N^3)$. 
\label{lemma: tilde Ws}
\end{lemma}
\begin{proof} 
 Recall that 
\begin{equation*}
    \widetilde  \W_s(v_i^s,v_j^s) =\widetilde \W (i, j)=  \sum_{k=1}^m \alpha^k_{ij} \W^k(v_i^k,v_j^k)  \mathbbm{1}(i,j \in \calV_k), i,j \in \calV_s.
\end{equation*}
So 
\begin{equation*}
    \widetilde \E_s(v_i^s,v_j^s) = \sum_{k=1}^m \alpha^k_{ij} \E^k(v_i^k,v_j^k) \mathbbm{1}(i,j \in \calV_s), i,j \in \calV_s.
\end{equation*}
Since $\E_s, s\in[m]$ are independent and recall that $\sigma = \max_{s\in [m]} \sigma_s$ and $$\sum_{k=1}^m \alpha^k_{ij} \mathbbm{1}(v_i^{k},v_j^{k} \in \calV_k)=1,$$
we have $\sum_{k=1}^m (\alpha^k_{ij})^2 \sigma_k^2 \mathbbm{1}(v_i^{k},v_j^{k} \in \calV_k) \leq \sigma^2$. Hence, $\widetilde \E_s$ has independent mean zero (upper triangular) entries with sub-Gaussian norm smaller than $\sigma$. Then 
$$\|\widetilde \E_s \| \lesssim \sqrt{n_s} \sigma,s \in [m]$$
with  probability $1-O(m/n_s^{6})$ by Theorem 4.4.5 of \cite{vershynin2018high} and the probability union bound. By Assumption \ref{assump: prob}, $n_s \geq p_0 N /2 \geq \sqrt{N}$, then $1/n_s^{6} \leq 1/N^3$. In addition,  $n_s \leq 3 p_s N/2$ leads to 
$$\| \widetilde \E_s \| \lesssim \sqrt{N p_0} \sigma,s \in [m] $$ 
with probability at least $1-O(m/N^3)$,  and  based on Assumption \ref{assump: signal to noise ratio}, we have $$\|\widetilde \E_s \| \ll  \frac{p_0 \lambda_{\min}}{4} \leq \lambda_r(\W_s^{\ast}), \ s \in [m].$$ 
\end{proof}

We then bound $\|\widetilde{\A} \Q_A - \A\|$ and $\|\widetilde{\B} \Q_B - \B\|$ for the case $m=2$ in the following lemma. 
\begin{lemma}
\label{lemma:A,B}
Based on the notation on Section \ref{sec: Imputation Error} with the assumptions that $\|\widetilde \E_l \| \ll \lambda_r(\W_l^{\ast})$ and $\tau_l =  \lambda_1(\W_l^{\ast})/\lambda_r(\W_l^{\ast}), l=s,k$ are bounded, we have
$$\|\widetilde{\A} \Q_A - \A\| \lesssim \frac{\tau_s}{\sqrt{\lambda_r(\W_s^{\ast})}} \|\widetilde \E_s\| \quad {\rm and} \quad \|\widetilde{\B} \Q_B - \B\|  \lesssim \frac{\tau_k}{\sqrt{\lambda_r(\W_k^{\ast})}} \| \widetilde \E_k\|.$$ 
\end{lemma}
\begin{proof}
 Define $\Q_s = \G(\widetilde \V_s^\top \V_s^{\ast})$, $\Q_k = \G(\widetilde \V_k^\top \V_k^{\ast})$ and recall that $\Q_A = \G(\widetilde \A^\top \A)$ and $\Q_B = \G(\widetilde \B^\top \B)$. The key decomposition we need is the following:
\begin{equation}
\widetilde{\A} \Q_A - \A = \widetilde{\A} (\Q_A - \Q_s) + \widetilde \V_s [ \widetilde \bSigma_s^{\frac{1}{2}} \Q_s - \Q_s  (\bSigma_s^{\ast})^{\frac{1}{2}} ] + (\widetilde \V_s \Q_s - \V^{\ast}_s) (\bSigma_s^{\ast})^{\frac{1}{2}}.
\label{dec}    
\end{equation}
For the spectral norm error bound, the triangle inequality together with (\ref{dec}) yields
$$\|\widetilde{\A} \Q_A - \A\| \leq \|\widetilde \bSigma_s^{\frac{1}{2}}\| \|\Q_A - \Q_s\| + \|\widetilde \bSigma_s^{\frac{1}{2}} \Q_s - \Q_s  (\bSigma_s^{\ast})^{\frac{1}{2}}\| + \sqrt{\lambda_1(\bSigma_s^{\ast})}\|\widetilde \V_s \Q_s - \V_s^{\ast}\|,$$ 
where we have also used the fact that $\|\widetilde \V_s\|=1$. Recognizing that $\|\widetilde \W_s - \W_s^{\ast}\| = \|\widetilde \E_s\| \ll \lambda_r(\W_s^{\ast})$ and the assumption that $\lambda_1(\W_s^{\ast})/\lambda_r(\W_s^{\ast})$ is bounded, we can apply Lemmas 47, 46, 45 of \cite{ma2018implicit}  
to obtain
\begin{equation*}
    \|\Q_A - \Q_s\| \lesssim \frac{1}{\lambda_r(\W_s^{\ast})} \|\widetilde \E_s\|,
\end{equation*}
\begin{equation*}
    \| \widetilde \bSigma_s^{\frac{1}{2}} \Q_s - \Q_s (\bSigma_s^{\ast})^{\frac{1}{2}}\| \lesssim \frac{1}{\sqrt{\lambda_r(\W_s^{\ast})}} \|\widetilde \E_s \|,
\end{equation*}
\begin{equation*}
    \| \widetilde \V_s \Q_s - \V_s^{\ast} \| \lesssim \frac{1}{\lambda_r(\W_s^{\ast})} \|\widetilde \E_s \|.
\end{equation*}
These taken collectively imply the advertised upper bound
$$\| \widetilde{\A} \Q_A - \A \| \lesssim  \frac{\sqrt{\lambda_1(\W_s^{\ast})}}{\lambda_r(\W_s^{\ast})} \| \widetilde \E_s \| + \frac{1}{\sqrt{\lambda_r(\W_s^{\ast})}} \|\widetilde \E_s \|  \lesssim \frac{\sqrt{\tau_s}}{\sqrt{\lambda_r(\W_s^{\ast})}} \|\widetilde \E_s\|,$$
where we also utilize the fact that $\|\widetilde \bSigma_s\| \leq \| \bSigma_s^{\ast} \| + \|\widetilde \E_s \| \leq 2 \|\bSigma_s^{\ast} \| = 2 \|\W_s^{\ast} \|$ and  $\lambda_1(\W_s^{\ast})/\lambda_r(\W_s^{\ast})$ is bounded. Similarly, we have 
$$\|\widetilde{\B} \Q_B - \B\| \lesssim \frac{\sqrt{\tau_k}}{\sqrt{\lambda_r(\W_k^{\ast})}} \|\widetilde \E_k \|.$$
Combined with the fact that $ \tau_l = \lambda_1(\W_l^{\ast})/\lambda_r(\W_l^{\ast}) \leq 6 r \mu_0 \tau,l=s,k$, we have 
$$\|\widetilde{\A} \Q_A - \A\| \lesssim \frac{\sqrt{ r \mu_0 \tau } }{\sqrt{\lambda_r(\W_s^{\ast})}} \|\widetilde \E_s\| \quad {\rm and} \quad \|\widetilde{\B} \Q_B - \B\|  \lesssim \frac{ \sqrt{ r \mu_0 \tau} }{\sqrt{\lambda_r(\W_k^{\ast})}} \| \widetilde \E_k\|.$$
\end{proof}

\subsection{Probability bound for submatrix}

\begin{lemma}
 Denote $\R \in \RR^{d \times d}$ for the square diagonal matrix whose $j$th diagonal entry is $y_j$, where $\{y_j\}_{j=1}^n$ is a sequence of independent $0-1$ random variables with common mean $p$. Let $\B \in \RR^{q \times d}$ with rank $r$ and $d > \max\{e^2, r^2\}$.
 \begin{itemize}
     \item If $p = o(1/\log d)$ or $p$ is bounded away from $0$ for all $d$,  we have
     \begin{equation}
    \PP \{ \| \B \R\| \geq C p^{\frac{1}{2}}  \| \B \| \} \leq \delta
    \end{equation}
    \item else, 
       \begin{equation}
    \PP \{ \| \B \R\| \geq C p^{\frac{1}{2}}  \sqrt{p \log d}  \| \B \| \} \leq \delta
    \end{equation}
 \end{itemize}
for some universal positive constant $C$ and $\delta = 1/d^3$.
\label{lemma: Probability Bound}
\end{lemma}
\begin{proof}
By Theorem 3.1 and 4.1 of \cite{tropp2008norms}, we have
\begin{equation}
   \EE_k \| \B \R \| \leq 6 \sqrt{\max \{k, 2\log r\}} \frac{p}{1-p} \max_{|T| \leq p^{-1}} [\sum_{j \in T} \|\b_j \|_2^k]^{1/k} + \sqrt{p} \|\B \|. 
\end{equation}
for $k \in [2,\infty)$ where $\EE_k \X = (\EE |\X|^k)^{1/k}$ and the $\ell_1$ to $\ell_2$ operator norm $\|\cdot\|_{1 \rightarrow 2}$ computes the maximum $\ell_2$ norm of a column. In addition, $\b_j$ is the $j$th column of $\B$ and $T \subset [d]$. Since $\|\b_j\|_2 \leq \|\B \|$, we have $$\max_{|T| \leq p^{-1}} [\sum_{j \in T} \|\b_j \|_2^k]^{1/k} \leq (p^{-1} \| \B \|^k)^{1/k} =  p^{-1/k} \| \B \|.$$
As a result, 
\begin{equation}
    \EE_k \| \B \R \| \leq p^{\frac{1}{2}} \{ \frac{6\sqrt{ \max \{k,2 \log r\}} p^{\frac{1}{2} - \frac{1}{k}} }{1-p} + 1 \} \| \B \|
\end{equation}
for $k \in [2,\infty)$. In addition, it is obviously that $\EE_k \|\B \R \| \leq \| \B \|$. When $p \geq \frac{1}{2}$, we have $$p^{\frac{1}{2}} \{ \frac{6 \sqrt{ \max \{k,2 \log r\}} p^{\frac{1}{2} - \frac{1}{k}} }{1-p} + 1 \} \geq p^{\frac{1}{2}} \{ 12 \sqrt{2 \log r}p^{\frac{1}{2} - \frac{1}{k}}  + 1 \} \geq \frac{1}{\sqrt{2}} \{ 12 \sqrt{ \log r} + 1 \} >1$$ 
and when $p < \frac{1}{2}$ we have 
$$p^{\frac{1}{2}} \{ \frac{6 \sqrt{ \max \{k,2 \log r \}}p^{\frac{1}{2} - \frac{1}{k}} }{1-p} + 1 \}  < p^{\frac{1}{2}} \{ 12 \max \sqrt{ \{k,2 \log r \}}p^{\frac{1}{2} - \frac{1}{k}} + 1 \}.$$
As a result, we have
$$\EE_k \|\B \R \| \leq c_1(p,r,k) \|\B\|$$
where $c_1(p,r,k) = \min \{1, p^{\frac{1}{2}} \{ 12 \sqrt{ \max \{k,2 \log r \}}p^{\frac{1}{2} - \frac{1}{k}} + 1 \} \}$. Let $k_0 = \log d \geq 2 \log r$. Then by Markov inequality, we have
\begin{equation}
    \PP \{ \| \B \R\| \geq p^{\frac{1}{2}}  \{ \delta^{-1/k_0} c_1(p,r,k_0)/\sqrt{p} \} \| \B \| \} \leq \delta.
\label{eq:Markov}
\end{equation}
We discuss the \eqref{eq:Markov} dependent on the conditions of $p$.

{\bf Case 1:} $0<p<c_3/\log d$ for all $d>0$ and some fixed constant $c_3>0$.  Then $\delta^{-1/q_0} = e^3$ is a constant. In addition, $ \sqrt{k_0} p^{\frac{1}{2}- \frac{1}{k_0}} \leq \sqrt{c_3} \{c_3/\log d\}^{-1/\log d} < c_4$ for some constant $c_4$ since $\lim\limits_{x \to \infty} x^{1/x} =1$ is bounded. As a result, $c_1(p,r,k_0)/\sqrt{p} \leq 12 c_4 + 1$ is also bounded.

{\bf Case 2:} $p \geq c_5$ for all $d>0$ and some fixed constant $0<c_5<1$.  Then let $c_6 = 1/\sqrt{c_5}$ and we have
\begin{equation}
    \PP \{ \| \B \R\| > p^{\frac{1}{2}}  c_6 \| \B \| \} \leq \delta
\end{equation}
since $\| \B \R\| \leq \| \B \|$ almost surely.

{\bf Case 3:} $p =g(d)/\log d$ for some function $g(d)>0$ which satisfies $\lim\limits_{d \to \infty} g(d) = \infty$ and $\lim\limits_{d \to \infty} g(d)/\log d = 0$. We still have $\delta^{-1/k_0} = e^3$. In addition, $ c_1(p,r,k_0)/\sqrt{p}\leq 12 \sqrt{ k_0}p^{\frac{1}{2} - \frac{1}{k_0}} + 1 \leq 12 \sqrt{g(d)} (\frac{\log d}{g(d)})^{1/\log d} + 1 \leq c_7 \sqrt{g(d)} = c_7 \sqrt{p \log d}$ for some constant $c_7$ since $(\log d/g(d))^{1/\log d}$ is bounded. 

Based on Case 1, 2 and 3, letting $C = \max\{e^3(12c_4+1),c_6,e^3 c_7\}$, we will get the result. 
\end{proof}

Let $c_1 = \lim_{N \to \infty} p_0$ and $c_2 = \lim_{N \to \infty} p_0 \log N$. Define
\begin{equation}
    f(p_0,N) = \mathbbm{1}(c_1>0 \ {\rm or} \ c_2=0) + \{1-\mathbbm{1}(c_1>0 \ {\rm or} \ c_2=0)\} \sqrt{ p_0 \log N}. 
\label{def: f(p0,N)}    
\end{equation}
Then we have the following proposition.

\begin{proposition}
Based on the definition of \eqref{def:A,B}, under the assumption that $p_0$ is bounded away from 1, e.g., $\lim_{N_0 \to \infty} p_0 < 1$, directly apply Lemma \ref{lemma: Probability Bound}, we will get 
\begin{equation}
\begin{aligned}
   \| \widetilde \A_1 \Q_A - \A_1 \| \lesssim  \sqrt{1-p_0} \| \widetilde \A \Q_A - \A\|; \quad &  \| \widetilde \A_2 \Q_A - \A_2\| \lesssim  \sqrt{p_0} f(p_0, N) \| \widetilde \A \Q_A - \A\|; \\
   \| \widetilde \B_2 \Q_B - \B_2\| \lesssim  \sqrt{1-p_0}  \| \widetilde \B \Q_B - \B \|; \quad &
   \| \widetilde \B_1 \Q_A - \B_1\| \lesssim  \sqrt{p_0} f(p_0, N) \| \widetilde \B \Q_B - \B\|; \\
   \|\widetilde  \A_1 \| \lesssim  \sqrt{1-p_0} \| \widetilde \A \|; \quad & \| \A_1 \| \lesssim   \sqrt{1-p_0}  \|  \A \|;\\
   \|  \A_2 \| \lesssim  \sqrt{p_0}  f(p_0, N) \| \A \|; \quad & \| \widetilde \B_1 \| \lesssim   \sqrt{p_0} f(p_0, N) \| \widetilde \B\|; \\
   \| \widetilde \B_2 \| \lesssim   \sqrt{1-p_0}  \| \widetilde \B\|; \quad & \| \B_2 \| \lesssim    \sqrt{1-p_0}  \|  \B\|;
\label{eq:subnorm}
\end{aligned}    
\end{equation}
with probability $1-10/N^3$.
\label{prop: Probability Bound}
\end{proposition}

\subsection{Orthogonal Procrustes problem}

\begin{lemma}[Orthogonal Procrustes problem]
Based on the definition of \eqref{def:A,B}, the condition of \eqref{eq:subnorm}, the Assumption \ref{assump: prob}, $\lambda_1(\W_l^{\ast}) \leq 3 p_0 r \mu_0/2 \lambda_{\max}$, $\lambda_r (\W^{\ast}_l) \geq p_l \lambda_{\min}/4$, and $\| \widetilde \E_l \| \ll \lambda_r (\W^{\ast}_l), l = s,k$, and $n_{s k} \geq 64  r \mu_0 \tau (\log r  + \log 2 N^3)$, we have 
\begin{equation}
 \|\Q_B^\top \widetilde \O \Q_A- \O \| \lesssim \frac{f(p_0, N)^2 r \mu_0 \tau}{ p_0 \lambda_{\min}} \{ \| \widetilde \E_s\| + \| \widetilde \E_k\|  \}
\end{equation}
with probability $1-2/N^3$. 
\label{lemma: orthogonal}
\end{lemma}

\begin{proof}
First, 
\begin{equation*}
\begin{aligned}
&  \| \A_2^\top \B_1 - \Q_A^\top \widetilde \A_2^\top \widetilde \B_1 \Q_B \| \leq \|\A_2\| \| \widetilde \B_1 \Q_B- \B_1 \| + \|\widetilde \B_1 \| \| \widetilde \A_2 \Q_A- \A_2 \| \\
& \leq  p_0  f(p_0, N)^2 \{\|\A \| \| \widetilde \B \Q_B- \B \| + \|\widetilde \B \| \| \widetilde \A \Q_A- \A \| \} \\
& \leq 2 p_0  f(p_0, N)^2 \{\|\A \| \| \widetilde \B \Q_B- \B \| + \| \B \| \| \widetilde \A \Q_A- \A \| \} \\
& \lesssim  p_0  f(p_0, N)^2  \{ \sqrt{ \frac{r \mu_0 \tau \lambda_1(\W_s^{\ast})}{\lambda_r(\W_k^{\ast})} } \| \widetilde \E_k\| +  \sqrt{ \frac{r \mu_0 \tau \lambda_1(\W_k^{\ast})}{\lambda_r(\W_s^{\ast})} } \| \widetilde \E_s\|  \} \\
& \leq p_0  f(p_0, N)^2 r \mu_0 \tau \{ \| \widetilde \E_s\| + \| \widetilde \E_k\|  \} 
\end{aligned}    
\end{equation*}
where the second inequality comes from \eqref{eq:subnorm}, the third inequality comes from  $\| \widetilde \B \| \leq \sqrt{\| \W_k^{\ast} \| + \|\widetilde \E_k \|} \leq \sqrt{2 \| \W_k^{\ast} \| } \leq 2 \| \B \|$ and the last inequality comes from  Lemma \ref{lemma:A,B}. In addition, since
\begin{equation*}
\begin{aligned}
\sigma_{r-1} (\A_2^\top \B_1) & \geq \sigma_r (\A_2^\top \B_1) = \sigma_r ( (\V_{s2}^{\ast} )^\top (\bSigma_s^{\ast})^{1/2} (\bSigma_k^{\ast})^{1/2} \V_{k1}^{\ast}) \\
& \geq \sigma_{\min}(\V_{s2}^{\ast})  \sqrt{ \lambda_r(\bSigma_s^{\ast}) \lambda_r( \bSigma_k^{\ast})} \sigma_{\min}(\V_{k1}^{\ast}) 
\end{aligned}    
\end{equation*}
and again by $p_0 \geq C \sqrt{\mu_0 r \tau \log N /N}$, we will have $p_0 \geq \sqrt{64 r \mu_0 \tau (\log r  + \log 2N^3 )/N}$. Then by Lemma \ref{lemma:lower bound U}, $\sigma_{\min}(\V_{s2}^{\ast}) \sigma_{\min}(\V_{k1}^{\ast}) \geq  p_0/6$ holds with probability $1-2/N^3$. Then
$$\sigma_{r-1} (\A_2^\top \B_1) \geq \sigma_r (\A_2^\top \B_1) \geq p_0^2 \lambda_{\min}/24.$$
So we can apply Lemma 23 of \cite{ma2018implicit} to get 
\begin{equation}
\begin{aligned}
 & \|\Q_A^\top \widetilde \O \Q_B- \O \| \leq \frac{\| \A_2^\top \B_1 - \Q_A^\top \widetilde \A_2^\top \widetilde \B_1 \Q_B \|}{\sigma_{r-1} ( \A_2^\top \B_1) + \sigma_r (\A_2^\top \B_1)}  \\
 & \leq \frac{p_0  f(p_0, N)^2 r \mu_0 \tau}{ 2 p_0^2 \lambda_{\min}/24} \{ \| \widetilde \E_s\| + \| \widetilde \E_k\|  \} \lesssim \frac{f(p_0, N)^2 r \mu_0 \tau}{ p_0 \lambda_{\min}} \{ \| \widetilde \E_s\| + \| \widetilde \E_k\|  \}.
\end{aligned}    
\end{equation}
\end{proof}


\end{document}